%% file: combined.tex
\documentclass[letterpaper, 10 pt, conference]{ieeeconf}  

\IEEEoverridecommandlockouts                              
\usepackage{cite}
\usepackage{graphicx}
\usepackage{textcomp}
\usepackage{xcolor}
\usepackage{amsmath, amsfonts, amssymb}

\usepackage{amsthm}
\usepackage{hyperref}
\makeatletter\g@addto@macro\UrlBreaks{\do\-}\makeatother 
\usepackage{algorithm}
\makeatletter
\renewcommand\fs@ruled{\def\@fs@cfont{\bfseries}\let\@fs@capt\floatc@ruled
  \def\@fs@pre{\vspace{5pt}\hrule height.8pt depth0pt \kern2pt}
  \def\@fs@post{\kern2pt\hrule\relax}%
  \def\@fs@mid{\kern2pt\hrule\kern2pt}%
  \let\@fs@iftopcapt\iftrue}
\makeatother
\usepackage[noend]{algpseudocode}
\usepackage{subcaption}
\usepackage{float}

\newtheorem{theorem}{Theorem}
\newtheorem{assumption}{Assumption}
\newtheorem{proposition}{Proposition}
\newtheorem{lemma}{Lemma}

\theoremstyle{definition}

\theoremstyle{remark}
\newtheorem{remark}{Remark}

\newcommand{\ind}{\mathbb{I}}                 
\newcommand{\R}{\mathbb{R}}
\newcommand{\E}{\mathbb{E}}
\newcommand{\co}{\mathrm{co}}                 
\newcommand{\norm}[1]{\lVert #1\rVert}
\newcommand{\Lib}{\mathcal{I}}                
\newcommand{\Cset}{\mathcal{C}}               
\newcommand{\Bset}{\mathcal{B}}               
\newcommand{\Ureg}{\mathcal{U}}               
\newcommand{\Tie}{\mathcal{T}}                
\newcommand{\Ebe}{E^{\mathrm{BE}}}            
\newcommand{\Hbar}{\overline{H}}
\newcommand{\dd}{\mathrm{d}}

\def\BibTeX{{\rm B\kern-.05em{\sc i\kern-.025em b}\kern-.08em
    T\kern-.1667em\lower.7ex\hbox{E}\kern-.125emX}}
\title{\LARGE \bf Full-Gradient Successor Feature Learning for Multi-Task Transfer in Reinforcement Learning}

\author{Ritish Shrirao\\
Dept. of CSE, IIIT Bangalore\\
\href{mailto:ritish.shrirao@iiitb.ac.in}{ritish.shrirao@iiitb.ac.in}
\and
Aditya Priyadarshi\\
Dept. of CSE, IIIT Bangalore\\
\href{mailto:aditya.priyadarshi@iiitb.ac.in}{aditya.priyadarshi@iiitb.ac.in}
\and
Raghuram Bharadwaj Diddigi\\
Dept. of DSAI, IIIT Bangalore\\
\href{mailto:raghuram.bharadwaj@iiitb.ac.in}{raghuram.bharadwaj@iiitb.ac.in}
}

\begin{document}

\maketitle
\thispagestyle{empty}
\pagestyle{empty}


\begin{abstract}
Successor Features (SFs) provide a principled framework for multi-task transfer in Reinforcement Learning (RL) by decoupling environment dynamics from reward functions. By learning representations of expected future feature occupancies, SFs enable knowledge acquired in one task to be transferred to other tasks with different reward functions. However, conventional SF learning relies on semi-gradient Temporal Difference (TD) updates, which, when combined with nonlinear function approximation, generally lack convergence guarantees and can exhibit instability. This motivates exploring alternative learning objectives that offer stronger theoretical properties. Inspired by Full-Gradient DQN, we propose Full-Gradient Successor Feature Representations Q-Learning (FG-SFRQL), which extends full-gradient optimization to successor feature learning by minimizing the squared SF Bellman residual and differentiating through both the prediction and the bootstrapped target. We establish almost-sure convergence of FG-SFRQL under randomized task sampling, characterize the objective it minimizes, and show that averaging the bootstrapped target over $N$ next-state transitions reduces double-sampling bias in proportion to $1/N$. Through experiments in both discrete- and continuous-state domains, we demonstrate that full-gradient learning achieves lower Bellman residuals and improved transfer performance compared to semi-gradient baselines, with FG-SFRQL attaining the highest final return across all five evaluated domains.\footnote{\label{fn:supp}Source code: \url{https://github.com/RitishShrirao/Full-Gradient-Successor-Feature-Representations}.}
\end{abstract}

\section{Introduction}
Deep Reinforcement Learning (RL)~\cite{sutton1998reinforcement}, with landmark results such as DQN~\cite{mnih2013playing}, enables agents to master high-dimensional control tasks. Most RL agents, however, are trained to optimize a single, fixed reward function, resulting in specialized policies that generalize poorly: when the task changes even slightly, they must often be retrained from scratch, wasting data and computation.

To address this limitation, researchers have explored representations that disentangle environment dynamics from task-specific rewards. A prominent approach is the Successor Representation (SR)~\cite{dayan1993improving}, which encodes the expected discounted future state occupancy under a given policy. Building on this idea, Successor Features (SFs)~\cite{barreto2017successor} extend the SR to feature-based representations, decomposing the value function into a predictive model of future feature activations (capturing dynamics) and a task-dependent weight vector (capturing rewards). In doing so, SFs provide a natural bridge between model-free and model-based RL~\cite{lehnert2020successor}: agents can predict future reward outcomes, as model-based planners do, while retaining the efficiency of model-free learning.

SFs have also been used for unsupervised pre-training~\cite{liu2021aps}, exploration~\cite{myers2024learning}, safety-constrained planning~\cite{gimelfarb2021risk} and inverse RL~\cite{jain2025non}.
Despite this widespread adoption, current deep SF methods largely rely on standard Q-learning mechanisms~\cite{mnih2013playing} to learn the successor features. From an optimization standpoint, these are \emph{semi-gradient} methods: they treat the bootstrapped target as a constant during the update, ignoring its dependence on the trainable parameters. While computationally convenient, semi-gradient methods combined with non-linear function approximation can diverge or oscillate~\cite{baird1995residual, tsitsiklis1996analysis}.
This issue is especially important for transfer, where SFs are paired with Generalized Policy Improvement (GPI)~\cite{barreto2017successor} to reuse knowledge across tasks. GPI evaluates previously learned policies under new reward functions and assumes that the learned SFs accurately capture expected future feature occupancy. If the optimization procedure fails to converge, these estimates become unreliable, undermining GPI's guarantees and degrading transfer performance.

To address these limitations, we introduce Full-Gradient Successor Feature Representations Q-Learning (FG-SFRQL). Drawing inspiration from the unified view of TD algorithms of~\cite{loth2006unified} and from Full-Gradient DQN~\cite{avrachenkov2021full}, we formulate SF learning as the direct minimization of a squared Bellman residual, differentiating through both the prediction and the target. Because single-sample residual minimization suffers from double-sampling bias in stochastic environments~\cite{baird1995residual}, we also introduce an averaged variant that averages the bootstrapped target over $N$ independent next-state transitions~\cite{avrachenkov2021full}.

Our contributions are as follows:
\begin{enumerate}
    \item We propose FG-SFRQL, an algorithm for learning successor features that replaces the semi-gradient update of SFR Q-learning (SFRQL)~\cite{reinke2021successor} with the full gradient of the squared SF Bellman residual, together with a target-averaged variant that reduces the double-sampling bias of single-sample residual minimization.
    \item Building on the analysis of Full-Gradient DQN~\cite{avrachenkov2021full}, we cast the randomized-task variants of FG-SFRQL as stochastic recursive inclusions, prove almost-sure convergence of the iterates to an internally chain transitive set of the limiting differential inclusion, and show that the limit consists of stationary points of an explicit aggregate SF Bellman objective whenever it is free of ties in the GPI and target-action selections.
    \item We show empirically that FG-SFRQL achieves lower Bellman residuals and outperforms semi-gradient baselines across all five domains with negligible computational overhead, and that averaging the target over $N$ next-state samples further boosts performance under stochastic transitions, as predicted by the analysis.
\end{enumerate}

\section{Related Work}
\subsection{Successor Features and Transfer Learning}
The Successor Representation (SR), introduced by Dayan~\cite{dayan1993improving}, decomposes the value function into a predictive occupancy map and a local reward vector. Barreto et al.~\cite{barreto2017successor} generalized this to Successor Features (SFs) and introduced Generalized Policy Improvement (GPI), establishing a framework for zero-shot transfer across tasks with shared dynamics. Reinke and Alameda-Pineda~\cite{reinke2021successor} extended them to non-linear rewards (SFR), the formulation we adopt.

\subsection{Gradient-Based Optimization and Stability in RL}
Semi-gradient Temporal Difference (TD) learning can diverge when combined with non-linear function approximation and off-policy sampling, the ``deadly triad''~\cite{sutton1998reinforcement, tsitsiklis1996analysis, bertsekas1995neuro}. Baird's residual algorithms~\cite{baird1995residual} perform true stochastic gradient descent on the Bellman residual; while stable, they were criticized for their double-sampling requirement and slow convergence.

Loth and Preux~\cite{loth2006unified} proposed a unified view of TD algorithms, showing that methods ranging from TD($\lambda$) to residual gradients can be understood as minimizing specific objective functions. More recently, Avrachenkov et al.~\cite{avrachenkov2021full} identified theoretical shortcomings of the standard DQN update through an o.d.e.\ analysis and proposed Full-Gradient DQN (FG-DQN), a modified and provably convergent stochastic approximation scheme. Our work extends the FG-DQN formalism to the multi-task successor-feature setting, deriving update rules that account for the vector-valued nature of SFs and for the max-over-policies operator inherent in GPI.

\subsection{Convergence Analysis of Successor Features}
The convergence of SFs under function approximation has received little analysis. Zhang et al.~\cite{zhang2024sf} analyze SF-DQN, but in the semi-gradient regime and under strong assumptions on initialization and exploration.

In contrast, we frame SF learning as a stochastic recursive inclusion driven by the full gradient of the SF Bellman residual. Because the update is a (generalized) gradient of a fixed objective, the analysis does not require controlling a moving target, and we obtain almost-sure convergence under standard stability and regularity conditions, which we state explicitly in Section~\ref{sec:assumptions}.

\section{Background}
\label{sec:background}

\subsection{Successor Feature Representations (SFR)}
Successor Feature Representations (SFR)~\cite{reinke2021successor} handle arbitrary, potentially non-linear rewards. For a policy $\pi$, the $\xi$-function is the cumulative discounted probability of observing feature $\phi \in \Phi$ in the future:
\begin{equation}
    \xi^\pi(s, a, \phi) = \sum_{k=0}^\infty \gamma^k P(\phi_{t+k} = \phi \mid s_t=s, a_t=a, \pi),
\end{equation}
where $\phi_t = \phi(s_t,a_t,s_{t+1})$ and $\gamma \in [0, 1)$ is the discount factor. If the reward $R(\phi)$ depends only on the features, the action-value function is recovered by aggregating the $\xi$-values:
\begin{equation}
    Q^\pi(s, a) = \sum_{\phi \in \Phi} \xi^\pi(s, a, \phi) R(\phi).
\end{equation}
This allows an agent to evaluate an existing policy $\pi$ under a \emph{new} reward function $R'$ by replacing $R$ with $R'$, without new environment interaction.

\subsection{Generalized Policy Improvement (GPI)}
The primary transfer mechanism in this framework is GPI. Suppose an agent has learned a library of policies $\Pi = \{\pi_1, \dots, \pi_n\}$ for previous tasks. When facing a new task with reward $R_{\text{new}}$, the agent evaluates \emph{all} stored policies using their $\xi$-functions:
\begin{equation}
    Q^{\pi_i}_{\text{new}}(s, a) = \sum_{\phi \in \Phi} \xi^{\pi_i}(s, a, \phi) R_{\text{new}}(\phi).
\end{equation}
GPI then acts greedily with respect to the best of these values:
\begin{equation}
    \pi_{\text{GPI}}(s) \in \arg\max_{a} \max_{i} Q^{\pi_i}_{\text{new}}(s, a).
\end{equation}
The GPI policy performs at least as well as every policy in the library~\cite{barreto2017successor}. Extensions such as constrained GPI~\cite{kim2022constrained} further refine this selection by incorporating constraints or approximate models.


\section{Algorithmic Setup and Notation}
\label{sec:setup}

We follow notations similar to~\cite{reinke2021successor}. We assume a finite state space $\mathcal{S}$, a finite action set $\mathcal{A}$ and a finite feature set $\Phi$. There are $m$ tasks indexed by $j\in\{1,\dots,m\}$. All tasks share the transition kernel $P(\cdot\mid s,a)$ and the feature map $\phi(s,a,s')\in\Phi$; they differ only in their reward $\mathcal{R}_j(\phi)$. For each task $j$ the agent maintains a parametric successor-feature approximator $\xi_j(s,a,\phi;\theta^{(j)})$ with $\theta^{(j)}\in\R^d$, and we collect all task parameters into the joint vector $\Theta = (\theta^{(1)},\dots,\theta^{(m)}) \in \R^{m d}$. We write $\gamma(s')=0$ if $s'$ is terminal and $\gamma(s')=\gamma$ otherwise.

The value of policy $j$ under the reward of task $i$ is
\begin{equation}
Q^{i}_{j}(s,a;\Theta) = \sum_{\phi\in\Phi} \xi_{j}(s,a,\phi;\theta^{(j)})\,\mathcal{R}_{i}(\phi).
\label{eq:Qij}
\end{equation}
At each update the current task $i$ provides a transition $(s,a,s')$. The algorithm selects a prior policy $c$ by GPI over a library $\Lib_i$ of policies, choosing from the set
\begin{equation}
\Cset(\Theta,s,i) = \arg\max_{j\in\Lib_i}\; \max_{a\in\mathcal{A}}\; Q^{i}_{j}(s,a;\Theta).
\label{eq:gpi_prior}
\end{equation}
In the sequential Algorithm~\ref{alg:fg-sfrql} the library holds the tasks seen so far, $\Lib_i=\{1,\dots,i\}$; in the randomized Algorithms~\ref{alg:fg-sfrql-random} and~\ref{alg:fg-sfrql-avg} it holds all $m$ tasks.

The algorithm updates the parameter block of the current task $i$ and, if $c\neq i$, the block of the prior $c$; all other blocks are left unchanged. As in SFQL~\cite{barreto2017successor}, the two blocks bootstrap with different next actions: the current task uses the GPI action $\arg\max_{a'}\max_{j\in\Lib_i} Q^{i}_{j}(s',a')$, while the prior $c$ uses its own greedy action $\arg\max_{a'} Q^{c}_{c}(s',a')$ under its own reward. We denote the set of these target actions for block $j$ by $\Bset^{i}_{j}(\Theta,s')$.

If a maximizer is not unique, ties are broken by a fixed rule (e.g., smallest index). The analysis also covers arbitrary tie-breaking by considering the convex hull of the update vectors of all maximizers (Section~\ref{sec:setvalued}). We call $\Theta$ \emph{regular} if no ties occur in any of the GPI and target-action maximizations, and denote the set of regular parameters by $\Ureg$ and the tie set by $\Tie=\R^{md}\setminus\Ureg$. In a neighborhood of a regular point, all selections are locally constant.

Algorithm~\ref{alg:fg-sfrql} is shown in compact form below; the full pseudocode of the three variants of FG-SFRQL is given in Appendix~D:
\begin{itemize}
    \item \textbf{Algorithm~\ref{alg:fg-sfrql}} is the sequential implementation used in the experiments, where tasks are learned one after another. It is a direct modification of SFRQL that replaces the semi-gradient update with the full gradient.
    \item \textbf{Algorithm~\ref{alg:fg-sfrql-random}} samples the task from a distribution $\rho$ and the state-action pair (the \emph{pivot}) from a task-dependent distribution $\mu_i$, so that the sampling is stationary, as required by the mean-field analysis.
    \item \textbf{Algorithm~\ref{alg:fg-sfrql-avg}} extends Algorithm~\ref{alg:fg-sfrql-random} by \emph{averaging} over $N$ next-state samples from the same pivot.
\end{itemize}

\section{Averaged Loss and Full Gradient}
\label{sec:loss}

\begin{algorithm}[t]
\caption{Joint Full-Gradient Model-Free SFRQL (FG-SFRQL); full pseudocode in Appendix~D}
\label{alg:fg-sfrql}
\footnotesize
\begin{algorithmic}[1]
\Require exploration rate $\epsilon$; step sizes $\{\alpha_t\}$; reward learning rate $\alpha_R$; features $\Phi$
\For{$i = 1$ to $m$}
    \State Initialize $\theta_i^R$ randomly (unless $\mathcal{R}_i$ is given)
    \State $\theta^{(i)}\leftarrow\theta^{(i-1)}$ ($\theta^{(1)}$ random)
    \For{$t = 1$ to num\_steps}
        \State At the start of an episode, $s_t\leftarrow$ initial state
        \State $c\leftarrow$ GPI prior~\eqref{eq:gpi_prior} at $s_t$, $\Lib_i=\{1,\dots,i\}$
        \State $a_t\leftarrow$ $\epsilon$-greedy action of $Q^i_c(s_t,\cdot)$
        \State Take $a_t$; observe $r_t$, $s_{t+1}$ and $\phi_t=\phi(s_t,a_t,s_{t+1})$
        \State $\theta_i^R\leftarrow\mathrm{SGD}(\alpha_R)$ step on $(r_t-\tilde{\mathcal{R}}_i(\phi_t))^2$
        \State $\gamma_t\leftarrow\gamma(s_{t+1})$; end the episode if $s_{t+1}$ is terminal
        \State $b_i\leftarrow\arg\max_{a'}\max_{k\le i}Q^i_k(s_{t+1},a')$ \Comment{GPI}
        \State $\theta^{(i)}\leftarrow\theta^{(i)}-\alpha_t\nabla_{\theta^{(i)}}\ell_i$ \Comment{full gradient}
        \If{$c\neq i$}
            \State $b_c\leftarrow\arg\max_{a'}Q^c_c(s_{t+1},a')$ \Comment{own greedy}
            \State $\theta^{(c)}\leftarrow\theta^{(c)}-\alpha_t\nabla_{\theta^{(c)}}\ell_c$
        \EndIf
        \State $s_t\leftarrow s_{t+1}$
    \EndFor
\EndFor
\end{algorithmic}
\footnotesize $Q^i_k$ is~\eqref{eq:Qij} with the learned reward $\tilde{\mathcal{R}}_i$; $\ell_j$ is~\eqref{eq:avg_loss} with $N=1$, pivot $(s_t,a_t)$, next state $s_{t+1}$ and target action $b_j$, and its full gradient is~\eqref{eq:avg_gradient}, with the factor 2 absorbed into $\alpha_t$.
\end{algorithm}

With stochastic transitions, minimizing the single-sample Bellman residual yields a biased objective. Following FG-DQN~\cite{avrachenkov2021full}, we therefore average the target over several next states. For a pivot $(s, a)$ and current task $i$, let $\mathcal{K} = \{s'_p\}_{p=1}^N$ be $N$ next states sampled from $P(\cdot\mid s,a)$, with features $\phi_p=\phi(s,a,s'_p)$ and discounts $\gamma_p=\gamma(s'_p)$. For each sample, a target action $b_{j,p}\in\Bset^{i}_{j}(\Theta,s'_p)$ is chosen as described above. We define the \emph{averaged} Bellman error of block $j$ at feature $\phi$ as
\begin{multline}
\bar\delta_j(\phi) = \frac{1}{N} \sum_{p=1}^N \Big[ \ind(\phi = \phi_p) + \gamma_p\, \xi_j(s'_p, b_{j,p}, \phi; \theta^{(j)}) \Big] \\
- \xi_j(s, a, \phi; \theta^{(j)}),
\end{multline}
and the averaged squared Bellman residual as
\begin{equation}
\ell_j = \sum_{\phi\in\Phi} \bar\delta_j(\phi)^2 .
\label{eq:avg_loss}
\end{equation}
This loss measures the distance between the current prediction and the \emph{empirical mean} of the target, which approaches the true expectation $\E[\cdot\mid s,a]$ as $N \to \infty$. For $N=1$ it is the single-sample squared residual used in Algorithms~\ref{alg:fg-sfrql} and~\ref{alg:fg-sfrql-random}; $N>1$ gives Algorithm~\ref{alg:fg-sfrql-avg}, and Algorithm~\ref{alg:fg-sfrql} with $N>1$ is the \emph{averaged variant} of Algorithm~\ref{alg:fg-sfrql} used in Section~\ref{subsec:slippery} and Appendix~C.5.

Following the Full-Gradient DQN scheme, we differentiate with respect to $\theta^{(j)}$ in \emph{both} the target and the prediction, holding the target actions fixed. With the averaged feature gradient
\begin{multline}
\bar D_j(\phi) = \frac{1}{N}\sum_{p=1}^N \gamma_p \nabla_{\theta^{(j)}} \xi_j(s'_p, b_{j,p}, \phi; \theta^{(j)}) \\
- \nabla_{\theta^{(j)}} \xi_j(s, a, \phi; \theta^{(j)}),
\end{multline}
the full gradient of the loss is
\begin{equation}
\nabla_{\theta^{(j)}} \ell_j = 2\sum_{\phi \in \Phi} \underbrace{\bar\delta_j(\phi)}_{\text{Bellman error}} \; \underbrace{\bar D_j(\phi)}_{\text{feature gradient}}.
\label{eq:avg_gradient}
\end{equation}
A semi-gradient method would keep only the second term of $\bar D_j(\phi)$.
The joint update vector $G$ uses this gradient in the appropriate blocks. If the current task is $i$ and the chosen prior is $c$, the $j$-th block of $G$ is
\begin{equation}
    G_j =
    \begin{cases}
        \nabla_{\theta^{(i)}} \ell_i & \text{if } j = i, \\
        \nabla_{\theta^{(c)}} \ell_c & \text{if } j = c \text{ and } c \neq i, \\
        \mathbf{0} & \text{otherwise.}
    \end{cases}
    \label{eq:joint_gradient_components}
\end{equation}
The target action is chosen separately for each sample $s'_p$: selecting a single action for the averaged vector $\frac1N\sum_p \xi(s'_p,\cdot)$ would replace $\E[\max(\cdot)]$ by $\max\E[\cdot]$, and the resulting target would no longer correspond to the SF Bellman operator.

\paragraph*{What the update minimizes}
Averaging the loss over tasks $i\sim\rho$, pivots $(s,a)\sim\mu_i$ and batches $\mathcal{K}$ gives the objective
\begin{equation}
E^{(N)}(\Theta) = \E\Big[\ell_i + \ind\{c\neq i\}\,\ell_c\Big].
\label{eq:objective}
\end{equation}
The following proposition relates this objective to the true Bellman error.

\begin{proposition}[Bias--variance decomposition]
\label{prop:decomposition}
For every $\Theta$ and $N\ge1$,
\begin{equation}
E^{(N)}(\Theta) = \Ebe(\Theta) + \tfrac{1}{N}\,V(\Theta),
\label{eq:decomposition}
\end{equation}
where $\Ebe$ is the aggregate mean-squared SF Bellman error (the objective~\eqref{eq:objective} with each per-sample target replaced by its expectation over $s'$), and $V\ge0$ is the corresponding conditional variance of the target.
\end{proposition}

The proof is given in Appendix~A (Lemma~\ref{lemma:mean-field}). For a single sample ($N=1$), the objective exceeds the Bellman error by the conditional variance of the target, which is the well-known double-sampling bias of residual methods~\cite{baird1995residual}. Averaging over $N$ samples (Algorithm~\ref{alg:fg-sfrql-avg}) reduces this bias to $V/N$, and $\Ebe$ is recovered in the limit $N\to\infty$. Under deterministic dynamics, $V=0$ and all $N$ samples are identical, so that averaging has no effect.

\paragraph*{Comparison with FG-DQN}
In FG-DQN, the bootstrapped target $\max_b Q(s',b;\theta)$ is a continuous function of $\theta$, since it is the maximum of finitely many continuous functions. In the SF setting, by contrast, the target is the \emph{vector} $\xi_j(s',b,\cdot)$ evaluated at a greedy action $b$, and this vector changes discontinuously whenever the greedy action switches. Consequently, the objective $E^{(N)}$ is smooth only away from the tie set $\Tie$. This motivates the set-valued analysis of Section~\ref{sec:setvalued} and the restriction of Theorem~\ref{thm:main}(ii) to tie-free limit sets.

\section{Assumptions}
\label{sec:assumptions}

\begin{assumption}[Properties of the system]\label{ass:regularity}
The sets $\mathcal{S}$, $\mathcal{A}$, $\Phi$ and the number of tasks $m$ are finite, $\gamma\in[0,1)$, and the task rewards $\mathcal{R}_i$ are known. For every $(s,a,\phi,j)$, the map $\theta\mapsto\xi_j(s,a,\phi;\theta)$ is twice continuously differentiable, so $\xi_j$ and its first two derivatives are bounded on compact sets.
\end{assumption}

\begin{assumption}[Sampling]\label{ass:sampling}
At each iteration $k$, given the past, the task is drawn as $i_k\sim\rho$, the pivot as $(s_k,a_k)\sim\mu_{i_k}$, and the batch $\mathcal{K}_k$ of $N_k$ next states i.i.d.\ from $P(\cdot\mid s_k,a_k)$. The distributions $\rho$ and $\mu_i$ do not depend on $\Theta$. The batch size is either fixed, $N_k\equiv N$, or grows, $N_k\to\infty$ (Algorithm~\ref{alg:fg-sfrql-random} is the case $N_k\equiv1$). We write $\mathcal{F}_k$ for the history up to iteration $k$.
\end{assumption}

\begin{assumption}[Step sizes]\label{ass:steps}
The step sizes satisfy $\alpha_k>0$, $\sum_k \alpha_k = \infty$ and $\sum_k \alpha_k^2 < \infty$.
\end{assumption}

\begin{assumption}[Stability of iterates]\label{ass:stability}
The iterates remain bounded almost surely: $\sup_k \norm{\Theta_k} < \infty$ a.s.
\end{assumption}

Let $E := E^{(N)}$ if $N_k\equiv N$ and $E := \Ebe$ if $N_k\to\infty$, and let $\Lambda$ be the set of regular stationary points of $E$, i.e., $\Theta\in\Ureg$ with $\nabla E(\Theta)=0$.

\begin{assumption}[Stationary points]\label{ass:critical}
The set of values $\{E(\Theta):\Theta\in\Lambda\}$ contains no interval; for example, $\Lambda$ is finite or countable.
\end{assumption}

Assumptions~\ref{ass:regularity}--\ref{ass:steps} are standard in stochastic approximation and RL literature~\cite{avrachenkov2021full, benaim2005stochastic, borkar2008stochastic, reinke2021successor}. Assumption~\ref{ass:stability} is also standard~\cite{borkar2008stochastic} and can be enforced by projecting the iterates onto a large ball. Assumption~\ref{ass:critical} is weaker than requiring finitely many stationary points, which is relevant for neural networks, whose symmetries typically give rise to non-isolated stationary points. In contrast to FG-DQN, it does not suffice here to assume that ties occur only on a set of parameters of measure zero, because the limiting dynamics can slide along such a set. We therefore state the final conclusion for limit sets that contain no ties.

\section{Set-Valued Mean Field}
\label{sec:setvalued}

Let $\omega=(i,s,a,\mathcal{K})$ denote one sample. When there are no ties, the update~\eqref{eq:joint_gradient_components} is a single vector. When there are ties, each admissible choice $\sigma$ of the prior $c$ and of the target actions $b_{j,p}$ gives an update vector $G^{\sigma}(\Theta,\omega)$, and we collect them in their convex hull. The single-sample and averaged mean-field maps are
\begin{align}
F^{(N)}(\Theta,\omega) &= \co\big\{G^{\sigma}(\Theta,\omega):\sigma\in\Sigma(\Theta,\omega)\big\}, \nonumber\\
\Hbar^{(N)}(\Theta) &= \E_{\omega}\big[F^{(N)}(\Theta,\omega)\big],
\label{eq:Hbar}
\end{align}
where $\Sigma(\Theta,\omega)$ is the finite set of admissible choices. The population map $\Hbar^{(\infty)}$ is defined in the same way, with the sample averages in~\eqref{eq:avg_loss}--\eqref{eq:avg_gradient} replaced by expectations under $P(\cdot\mid s,a)$.

By Assumption~\ref{ass:regularity}, each $G^\sigma$ is continuous in $\Theta$, and a small perturbation of $\Theta$ can only remove ties, not create new ones. Therefore $\Hbar^{(N)}$ is nonempty, compact, convex-valued, upper semicontinuous and bounded on compact sets. On the regular set $\Ureg$ the maps reduce to single gradients: $\Hbar^{(N)}=\{\nabla E^{(N)}\}$ and $\Hbar^{(\infty)}=\{\nabla\Ebe\}$ (Lemma~\ref{lemma:mean-field}, Appendix~A).

\section{Main Result}
\label{sec:theory}
We adapt the o.d.e.\ analysis of Full-Gradient DQN~\cite{avrachenkov2021full} to successor features, where (i)~the Bellman target is vector-valued and evaluated at a greedy action, (ii)~GPI maximizes over both actions and policies, and (iii)~the updates are coupled across tasks, using set-valued stochastic approximation~\cite{benaim2005stochastic, borkar2008stochastic}.

\begin{theorem}[Convergence of FG-SFRQL]
\label{thm:main}
Let $\{\Theta_k\}$ be generated by Algorithm~\ref{alg:fg-sfrql-random} or~\ref{alg:fg-sfrql-avg}, and let Assumptions~\ref{ass:regularity}--\ref{ass:stability} hold. Let $\Hbar=\Hbar^{(N)}$ if $N_k\equiv N$ and $\Hbar=\Hbar^{(\infty)}$ if $N_k\to\infty$. Then, almost surely:
\begin{enumerate}
\item[(i)] the iterates converge to a compact, connected, internally chain transitive (ICT) set $L$ of the differential inclusion
\begin{equation}
\dot\Theta(t)\in-\Hbar(\Theta(t));
\label{eq:di}
\end{equation}
\item[(ii)] if, in addition, Assumption~\ref{ass:critical} holds and $L$ contains no ties, i.e., $L\subseteq\Ureg$, then $L$ consists of stationary points of $E$, $E(\Theta_k)$ converges and $\Theta_k$ approaches $\Lambda$. If $\Lambda$ is finite, $\Theta_k\to\Theta^\ast$ with $\nabla E(\Theta^\ast)=0$.
\end{enumerate}
\end{theorem}

Informally, an ICT set is invariant under the limiting dynamics, and any two of its points can be connected by trajectories that remain in the set. Theorem~\ref{thm:main} thus states that FG-SFRQL asymptotically behaves like gradient descent on $E$ wherever the GPI selections are unambiguous.

The proof relies on the following results on stochastic approximation with set-valued maps.

\begin{lemma}[Bena\"{\i}m, Hofbauer and Sorin~\cite{benaim2005stochastic}]
\label{lem:bhs}
Let $F$ be a set-valued map on $\R^{md}$ whose graph is closed, whose values $F(x)$ are nonempty, compact and convex, and which satisfies $\sup_{z\in F(x)}\norm{z}\le K(1+\norm{x})$ for some $K>0$. Let $\{x_k\}$ satisfy
\[
x_{k+1}-x_k-\alpha_k U_{k+1}\in\alpha_k F(x_k),
\]
where $\alpha_k>0$, $\sum_k\alpha_k=\infty$, $\alpha_k\to0$, and $U_{k+1}$ are (random) perturbations. Let $t_k=\sum_{n<k}\alpha_n$, and let $w:[0,\infty)\to\R^d$ be the piecewise-linear interpolation of the iterates on this time scale,
\[
w(t)=x_k+\frac{t-t_k}{\alpha_k}\,(x_{k+1}-x_k),\qquad t\in[t_k,t_{k+1}],
\]
so that $w(t_k)=x_k$ and $w$ is continuous. Its limit set is $L(w)=\bigcap_{t\ge0}\overline{\{w(s):s\ge t\}}$, the set of all limits of $w(s_n)$ along sequences $s_n\to\infty$.
\begin{enumerate}
\item[(a)] If $\sup_k\norm{x_k}<\infty$ and, for every $T>0$,
\[
\lim_{n\to\infty}\ \sup\Big\{\Big\|\sum_{i=n}^{k-1}\alpha_i U_{i+1}\Big\| : k>n,\ t_k\le t_n+T\Big\}=0,
\]
then $L(w)$ is internally chain transitive for $\dot x\in F(x)$.
\item[(b)] Every internally chain transitive set $L$ is invariant: through every point of $L$ passes a solution of $\dot x\in F(x)$ that stays in $L$ for all times.
\item[(c)] Let $L$ be an internally chain transitive set, let $\Lambda\subseteq\R^{md}$, and let $V$ be a function that is continuous on $L$. Suppose that for every solution $x$ and every $t>0$ with $x([0,t])\subseteq L$, $V(x(t))<V(x(0))$ if $x(0)\notin\Lambda$ and $V(x(t))\le V(x(0))$ if $x(0)\in\Lambda$. If $V(\Lambda)$ has empty interior, then $L\subseteq\Lambda$ and $V$ is constant on $L$ (proof in Appendix~A).
\end{enumerate}
\end{lemma}

\begin{proof}
The proof follows the o.d.e.\ approach for stochastic approximation, adapted to set-valued maps and differential inclusions~\cite{benaim2005stochastic, borkar2008stochastic}.

\paragraph*{\textbf{Update rule}}
The update is $\Theta_{k+1} = \Theta_k - \alpha_k \tilde{G}_k$, where $\tilde G_k$ has at most two non-zero blocks, for the current task $i_k$ and the GPI-selected prior $c_k$. We split $\tilde G_k$ into its conditional mean and a noise term. By Lemma~\ref{lemma:mean-field}, the conditional mean $h(\Theta_k)=\E[\tilde{G}_k \mid \mathcal{F}_k]$ is a selection from $\Hbar^{(N_k)}(\Theta_k)$. Let $M_{k+1} = \tilde{G}_k - h(\Theta_k)$ be the noise and, when $N_k\to\infty$, let $\varepsilon_k$ be the difference between the batch and population mean fields ($\varepsilon_k=\mathbf 0$ for fixed $N$). Then
\begin{equation}
\Theta_{k+1} \in \Theta_k - \alpha_k \big( \Hbar(\Theta_k) + \varepsilon_k + M_{k+1} \big),
\label{eq:sri}
\end{equation}
which is a stochastic recursive inclusion.

\paragraph*{\textbf{Noise and bias}}
By Lemma~\ref{lem:martingale_proof} (Appendix~A), $M_{k+1}$ is a martingale difference sequence, and for every $R>0$ there is a constant $C_R<\infty$, depending only on $R$ and the bounds of Assumption~\ref{ass:regularity}, such that on $\{\norm{\Theta_k}\le R\}$ its conditional second moment is at most $C_R$. Since $\ind\{\norm{\Theta_k}\le R\}$ is $\mathcal F_k$-measurable, $\sum_k\alpha_kM_{k+1}\ind\{\norm{\Theta_k}\le R\}$ is a martingale whose increments have second moments summing to at most $C_R\sum_k\alpha_k^2<\infty$; it is bounded in mean square and converges almost surely. On the event $\{\sup_k\norm{\Theta_k}\le R\}$ it equals $\sum_k\alpha_kM_{k+1}$. By Assumption~\ref{ass:stability}, the union of these events over $R\in\mathbb N$ has probability one, and a countable union of null sets is null. Hence the noise series $\sum_k\alpha_k M_{k+1}$ converges almost surely. By Lemma~\ref{lemma:mean-field}(c), $\norm{\varepsilon_k}=O(1/N_k)\to0$. Since the iterates are bounded (Assumption~\ref{ass:stability}), it suffices to argue on the event that they stay in a fixed ball, where $\Hbar$ is bounded.

\paragraph*{\textbf{Limiting differential inclusion}}
Write~\eqref{eq:sri} as $\Theta_{k+1}-\Theta_k-\alpha_k U_{k+1}\in\alpha_k F(\Theta_k)$ with $F:=-\Hbar$ and $U_{k+1}:=-(M_{k+1}+\varepsilon_k)$, which is the recursion of Lemma~\ref{lem:bhs}. On the event that the iterates stay in the closed ball $B_R=\{\Theta:\norm{\Theta}\le R\}$, replace $\Hbar$ by $\Hbar\circ\pi_R$, where $\pi_R$ is the projection onto $B_R$; this changes neither the iterates nor the dynamics on $B_R$. The resulting map has a closed graph, nonempty compact convex values and bounded range, so it satisfies the conditions on $F$ in Lemma~\ref{lem:bhs}.

The perturbations satisfy the condition of Lemma~\ref{lem:bhs}(a), i.e., $\sum_{i=n}^{k-1}\alpha_i U_{i+1}\to0$ uniformly over windows $t_k-t_n\le T$. For the noise this holds because $\sum_k\alpha_k M_{k+1}$ converges, so its tail sums vanish. For the bias it holds because the window sum is at most $(T+\alpha_n)\sup_{i\ge n}\norm{\varepsilon_i}\to0$. Since the iterates are bounded, Lemma~\ref{lem:bhs}(a) shows that the limit set $L:=L(w)$ of the interpolation $w$ of $\{\Theta_k\}$ (with $x_k=\Theta_k$) is internally chain transitive for $\dot\Theta\in F(\Theta)$. Two further steps are needed to obtain statement~(i).

 First, $L$ is also the set of limit points of the iterates $\{\Theta_k\}$. The increments vanish: $\Theta_{k+1}-\Theta_k=-\alpha_k\big(h(\Theta_k)+\varepsilon_k+M_{k+1}\big)$, where $h(\Theta_k)$ is bounded on $B_R$, $\varepsilon_k\to0$, $\alpha_k\to0$, and $\alpha_kM_{k+1}\to0$ because it is the general term of a convergent series. The interpolation passes through every iterate, $w(t_k)=\Theta_k$, so every limit point of $\{\Theta_k\}$ lies in $L$. Conversely, for $t\in[t_k,t_{k+1}]$ the point $w(t)$ lies on the segment from $\Theta_k$ to $\Theta_{k+1}$, so $\norm{w(t)-\Theta_k}\le\norm{\Theta_{k+1}-\Theta_k}\to0$. Hence, if $w(s_n)\to x$ along some $s_n\to\infty$, the iterates at the corresponding indices also converge to $x$, and every point of $L$ is a limit point of $\{\Theta_k\}$. As the limit set of a bounded continuous path,  $L$ is compact and connected.

Second, $L$ is ICT for the original inclusion~\eqref{eq:di}, not only for the modified one. All iterates lie in the closed ball $B_R$, so $L\subseteq B_R$. On $B_R$ the projection is the identity, $\pi_R(\Theta)=\Theta$, so the modified map $-\Hbar\circ\pi_R$ equals $-\Hbar$ there. Internal chain transitivity of $L$ involves only solutions that stay in $L$: by definition, any two points of $L$ are joined by chains of solution segments that lie in $L$, with arbitrarily small jumps between segments. On $L$ the two maps are equal, so these segments are also solutions of~\eqref{eq:di}, and $L$ is an ICT set of~\eqref{eq:di}. Here ICT for~\eqref{eq:di} is meant in this sense of chains of solution segments in $L$, which does not require solutions of~\eqref{eq:di} to exist for all times; the remaining steps apply Lemma~\ref{lem:bhs} to the modified map, which satisfies its conditions. This proves (i).

\paragraph*{\textbf{Convergence analysis}}
Suppose now that $L$ contains no ties, i.e., $L\subseteq\Ureg$. Since $\Ureg$ is open and $L$ is compact, $L$ has a tie-free neighborhood, on which the selections are fixed, $E$ is twice continuously differentiable, and $\Hbar(\Theta)=\{\nabla E(\Theta)\}$. There, the differential inclusion reduces to the o.d.e.
\begin{equation}
     \dot\Theta = -\nabla E(\Theta).
\end{equation}
Along its trajectories, the aggregate error decreases:
\begin{align}
\frac{\dd}{\dd t} E(\Theta(t)) &= \langle \nabla E(\Theta), \dot\Theta \rangle \nonumber\\
&= \sum_{j=1}^m \big\langle \nabla_{\theta^{(j)}} E(\Theta), -\nabla_{\theta^{(j)}} E(\Theta) \big\rangle \nonumber\\
&= -\sum_{j=1}^m \norm{\nabla_{\theta^{(j)}} E(\Theta)}^2 \;\le\; 0,
\label{eq:derivative}
\end{align}
with equality only at stationary points. Hence $E$ is a Lyapunov function for the limiting dynamics. On $L$ the modified map equals $-\nabla E$, so every solution segment that lies in $L$ solves $\dot\Theta=-\nabla E(\Theta)$. By~\eqref{eq:derivative}, $E$ decreases strictly along such a segment if it starts outside $\Lambda$, and does not increase if it starts in $\Lambda$. Since $E$ is also continuous on the tie-free neighborhood of $L$, it satisfies the conditions of Lemma~\ref{lem:bhs}(c). Since the values of $E$ on $\Lambda$ contain no interval (Assumption~\ref{ass:critical}), the lemma gives $L\subseteq\Lambda$ with $E$ constant on $L$. As the iterates approach $L$ and $E$ is continuous near $L$, $E(\Theta_k)$ converges and $\Theta_k$ approaches $\Lambda$. Finally, if $\Lambda$ is finite, the connected set $L$ must be a single point $\Theta^\ast$, and $\Theta_k\to\Theta^\ast$.
\end{proof}

The $j$-th block of $\nabla E$ collects the iterations in which $j$ is the current task \emph{and} those in which it is the GPI-selected prior (Remark~\ref{rem:block}, Appendix~A).

\paragraph*{Replay buffers and batch schedules}
With a replay buffer, the sampling laws are the empirical buffer laws rather than $\mu$ and $P$; the resulting error can be absorbed into $\varepsilon_k$ whenever the buffer distribution converges, e.g., under a fixed ergodic behavior policy. For growing batches, only $\varepsilon_k\to0$ is needed, with no condition linking $N_k$ to $\alpha_k$.

\section{Experiments}
\label{sec:exp}
\begin{figure*}[!t]
\vspace{0.3cm}
\centering
\includegraphics[width=\textwidth]{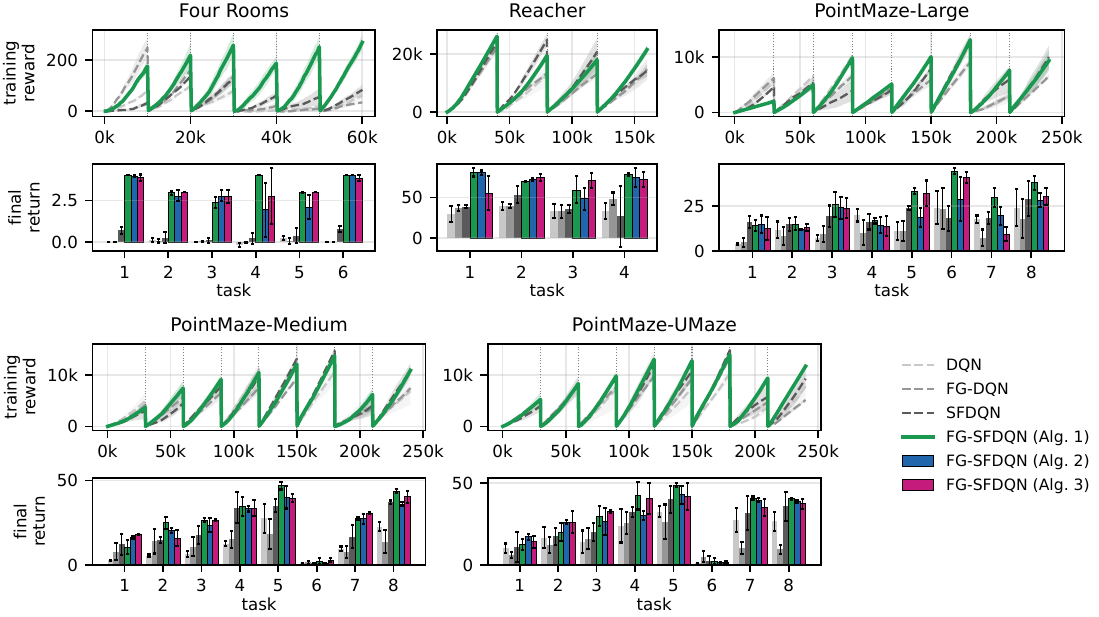}
\caption{Results in the five domains (mean\,$\pm$\,std over 3 seeds). For each domain, upper panel: cumulative training reward on the active task for the sequential methods; dotted lines mark the task switches, where the curves restart at zero. Lower panel: final evaluation per task after training on all tasks (mean return over 10 episodes of at most 100 steps).}
\label{fig:results}
\end{figure*}

\subsection{Experimental Setup}

We evaluate FG-SFRQL in five domains. In each domain, all tasks share the transition dynamics and differ only in their reward functions; Appendix~B.1 gives the details.
\begin{itemize}
    \item \textbf{Four Rooms}~\cite{chevalier2023minigrid} (discrete features): a grid world in which the agent collects objects on its way to a goal; the 6 tasks assign different values to the objects.
    \item \textbf{Reacher}~\cite{coumans2016pybullet} (continuous features): a two-joint arm with 9 discretized torques; each of the 4 tasks has a different target location.
    \item \textbf{PointMaze}~\cite{gymnasium_robotics2023github} (continuous features): a point agent with 9 discretized actions in the \texttt{UMaze}, \texttt{Medium} and \texttt{Large} layouts; each of the 8 tasks has a different goal location and a dense reward.
\end{itemize}
For the continuous-feature domains, whose rewards are linear in $\phi$, we use linear SFs $\psi_j(s,a)\in\R^{\dim\phi}$ (Appendix~B.1), demonstrating that FG-SFRQL generalizes effectively to continuous state spaces beyond the finite-space setting of Assumption~\ref{ass:regularity}.

\subsection{Baselines, Protocol and Hyperparameters}
We denote by SFDQN and FG-SFDQN the deep-network implementations of SFRQL and FG-SFRQL, respectively; unless stated otherwise, FG-SFDQN refers to Algorithm~\ref{alg:fg-sfrql}. We compare:
\begin{enumerate}
    \item \textbf{DQN}, a single network trained on the tasks in sequence, and \textbf{FG-DQN}~\cite{avrachenkov2021full}, its full-gradient version;
    \item \textbf{SFDQN}, successor features with semi-gradient updates and GPI; like Algorithm~\ref{alg:fg-sfrql}, it updates the current task on every sample and each GPI-selected prior on the samples where it is selected~\cite{barreto2017successor};
    \item \textbf{FG-SFDQN (ours)}: the sequential Algorithm~\ref{alg:fg-sfrql}, and the randomized variants Algorithm~\ref{alg:fg-sfrql-random} (single sample) and Algorithm~\ref{alg:fg-sfrql-avg} (averaging over $N=5$ next states).
\end{enumerate}
The sequential methods train on each task for a fixed number of steps before proceeding to the next one. Algorithms~\ref{alg:fg-sfrql-random} and~\ref{alg:fg-sfrql-avg} use the same total budget (steps per task $\times$ number of tasks), but draw the task uniformly at random at every step. Since the replay buffers of the DQN baselines store rewards, they are emptied at every task switch, whereas the SF methods store task-independent features and keep their buffers.

Hyperparameters are listed in Table~\ref{tab:hyperparams} (Appendix~B.2). Every configuration is run with 3 seeds, and we report mean\,$\pm$\,standard deviation across seeds. After training on all tasks, each trained agent is evaluated on every task for 10 episodes of at most 100 steps; SF methods act by GPI over their learned successor features. Following standard deep RL practice, all methods use a constant learning rate and minibatch updates.

\subsection{Results and Analysis}
\subsubsection{Training performance}
\label{subsec:comparison}
As shown in the upper panels of Fig.~\ref{fig:results}, FG-SFDQN (Algorithm~\ref{alg:fg-sfrql}) achieves the highest cumulative training reward across all five domains: it collects $2.7\times$ more reward than its semi-gradient counterpart SFDQN in Four Rooms ($1365$ vs.\ $500$) and $9$--$13\%$ more across all three PointMaze layouts, while matching SFDQN's high training return in Reacher ($85.3\text{k}$ vs.\ $84.1\text{k}$; full totals in Appendix~C.1). As expected, sequential training (Algorithm~\ref{alg:fg-sfrql}) collects more reward during training than the randomized variants (Algorithms~\ref{alg:fg-sfrql-random} and~\ref{alg:fg-sfrql-avg}), since sequential training warm-starts $\theta^{(i)}\leftarrow\theta^{(i-1)}$ on a focused task whereas uniform task sampling distributes early exploration across all $m$ tasks simultaneously.

\subsubsection{Final evaluation}
\label{subsec:final_eval}
The lower panels of Fig.~\ref{fig:results} show the final evaluation return on every task after training completes (mean values are given in Appendix~C.1). FG-SFDQN (Algorithm~\ref{alg:fg-sfrql}) achieves the highest mean return in all five domains, substantially outperforming all baselines: it improves over SFDQN by $8.2\times$ in Four Rooms ($3.38$ vs.\ $0.41$), by $85\%$ in Reacher ($71.71$ vs.\ $38.68$), and by $24$--$39\%$ across the three PointMaze layouts. Crucially, the randomized variants (Algorithms~\ref{alg:fg-sfrql-random} and~\ref{alg:fg-sfrql-avg}), which are directly covered by Theorem~\ref{thm:main}, rank second and third in every domain, outperforming SFDQN, DQN, and FG-DQN across all benchmarks. Single-network DQN baselines suffer from catastrophic forgetting across task switches (e.g., FG-DQN drops to $0.02$ in Four Rooms), whereas FG-SFDQN reduces the final squared SF Bellman residual by $5$--$46\times$ compared to SFDQN, yielding accurate SFs for reliable GPI transfer (Appendix~C.4).

\subsubsection{Effect of averaging}
\label{subsec:slippery}
Averaging reduces the double-sampling bias $V$, which is non-zero only when the next state of a pivot is random (Proposition~\ref{prop:decomposition}). Consistent with Proposition~\ref{prop:decomposition}, in the deterministic domains ($V=0$) target averaging with $N\in\{5,10,20\}$ matches single-sample FG-SFDQN while reducing standard deviation by up to $8\times$ (Appendix~C.5). We therefore evaluate averaging on \emph{slippery} variants of Four Rooms and Reacher, in which the action of the agent is replaced, with probability $p\in\{0.2,0.3\}$, by a uniformly random action. We compare Algorithm~\ref{alg:fg-sfrql} and Algorithm~\ref{alg:fg-sfrql-random} with their averaged counterparts ($N=5$; Algorithm~\ref{alg:fg-sfrql-avg} for Algorithm~\ref{alg:fg-sfrql-random}) and with the baselines, using 3 seeds, 6 tasks of 30{,}000 steps in Four Rooms, and otherwise the settings of Table~\ref{tab:hyperparams}. Since continuous states do not repeat in the buffer, Reacher draws the $N$ next states from the simulator.

\begin{figure*}[!t]
\centering
\begin{subfigure}[b]{0.33\textwidth}
\centering
\includegraphics[width=\linewidth]{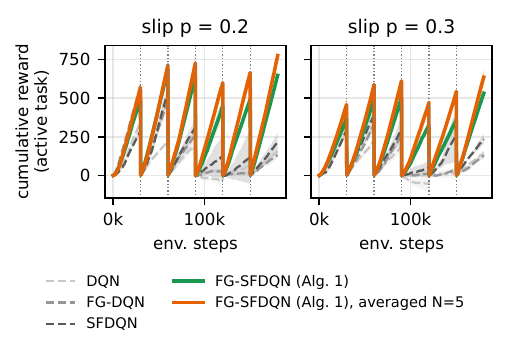}
\caption{Four Rooms: training reward per task}
\end{subfigure}\hfill
\begin{subfigure}[b]{0.62\textwidth}
\centering
\includegraphics[width=\linewidth]{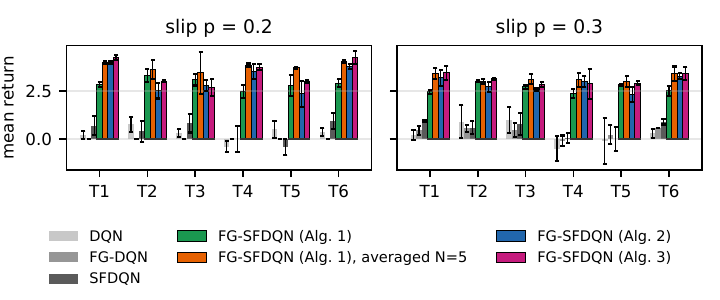}
\caption{Four Rooms: final evaluation per task}
\end{subfigure}

\medskip
\begin{subfigure}[b]{0.33\textwidth}
\centering
\includegraphics[width=\linewidth]{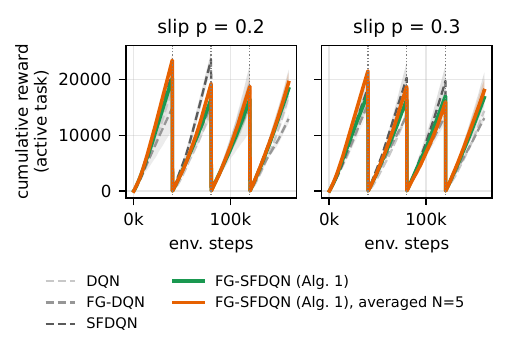}
\caption{Reacher: training reward per task}
\end{subfigure}\hfill
\begin{subfigure}[b]{0.62\textwidth}
\centering
\includegraphics[width=\linewidth]{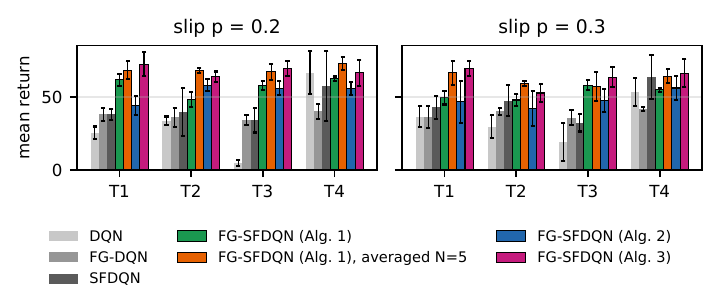}
\caption{Reacher: final evaluation per task}
\end{subfigure}
\caption{Slippery Four Rooms (top) and Reacher (bottom), mean\,$\pm$\,std over 3 seeds. Panels~(a) and~(c) show the cumulative training reward of the sequential methods; dotted lines mark the task switches, where the curves restart at zero. Totals are given in Appendices~C.2 and~C.3.}
\label{fig:slippery}
\end{figure*}

Under stochastic transitions ($V>0$), target averaging consistently improves both sequential and randomized FG-SFDQN (Fig.~\ref{fig:slippery}; full tables in Appendices~C.2--C.3). In slippery Four Rooms (Fig.~\ref{fig:slippery}a--b), averaging ($N=5$) increases cumulative training reward by $19$--$28\%$ for Algorithm~\ref{alg:fg-sfrql} and $12$--$15\%$ for Algorithm~\ref{alg:fg-sfrql-random}, while boosting final evaluation returns by $19$--$30\%$ (reaching $3.79$ at $p=0.2$ and $3.19$ at $p=0.3$). All four FG-SFDQN variants achieve final returns above $2.67$, whereas all baselines collapse below $0.55$. In slippery Reacher (Fig.~\ref{fig:slippery}c--d), averaging improves final evaluation return by $17$--$21\%$ for Algorithm~\ref{alg:fg-sfrql} (reaching $69.3$ and $61.7$) and by $27$--$31\%$ for Algorithm~\ref{alg:fg-sfrql-random} (reaching $68.1$ and $63.0$). While SFDQN collects comparable training reward in slippery Reacher, its semi-gradient estimates degrade across tasks, resulting in poor final retention ($42.3$ and $46.6$); averaged FG-SFDQN outperforms SFDQN in final evaluation by $32$--$64\%$.

\subsubsection{Computational overhead}
Full-gradient differentiation adds no per-step computational overhead over semi-gradient SFDQN: on PointMaze-UMaze, Algorithm~\ref{alg:fg-sfrql} matches the mean step time of SFDQN ($4.34$ vs.\ $4.47$\,ms) with a lower median ($2.67$ vs.\ $4.19$\,ms), while $N=5$ target averaging adds only $18\%$ overhead ($5.11$\,ms mean, $3.85$\,ms median; Appendix~C.6). Overall, FG-SFRQL substantially outperforms semi-gradient baselines in both Bellman residual reduction and multi-task transfer returns while incurring negligible computational overhead.

\IfFileExists{IEEEtran.bst}{%
    \bibliographystyle{IEEEtran}%
}{%
    \bibliographystyle{plain}%
}
\bibliography{references}

\clearpage
\section*{Appendix A: Proofs of the Lemmas}

Throughout, $\mathcal{F}_k$ is as in Assumption~\ref{ass:sampling}, $\omega_k=(i_k,s_k,a_k,\mathcal{K}_k)$ is the sample used at iteration $k$, and ties are broken by the fixed deterministic rule of Section~\ref{sec:setup}, so that $\tilde G_k=G(\Theta_k,\omega_k)$ for a measurable map $G$. For $R>0$ let
\begin{align*}
K_{1,R} &:= \max_{j,s,a,\phi}\ \sup_{\norm{\theta}\le R}|\xi_j(s,a,\phi;\theta)|,\\
K_{2,R} &:= \max_{j,s,a,\phi}\ \sup_{\norm{\theta}\le R}\norm{\nabla_\theta\xi_j(s,a,\phi;\theta)},
\end{align*}
which are finite by Assumption~\ref{ass:regularity}. They bound the $\xi$-values and their gradients while the parameters stay in the ball of radius $R$.

\paragraph*{Proof of Lemma~\ref{lem:bhs}(c)}
We follow~\cite[Prop.~3.27]{benaim2005stochastic} but use $V$ only on $L$.

\emph{Setting.} Recall~\cite[Sec.~3.3]{benaim2005stochastic} that a compact set $L$ is internally chain transitive if for all $x,y\in L$, $\varepsilon>0$ and $T>0$ there are solutions $x_1,\dots,x_n$ of $\dot x\in F(x)$ and times $t_1,\dots,t_n>T$ such that
\begin{itemize}
    \item $x_i([0,t_i])\subseteq L$ for all $i$;
    \item $\norm{x_i(t_i)-x_{i+1}(0)}\le\varepsilon$ for $i<n$;
    \item $\norm{x_1(0)-x}\le\varepsilon$ and $\norm{x_n(t_n)-y}\le\varepsilon$.
\end{itemize}
We call this an $(\varepsilon,T)$-chain from $x$ to $y$. We use two facts:
\begin{itemize}
    \item[(F1)] \textbf{\boldmath $V$ is nonincreasing along every solution segment in $L$}, since a time-shifted solution is again a solution and the hypothesis applies to its sub-segments;
    \item[(F2)] \textbf{the solutions with initial points in a compact set form a compact set of paths}, with the topology of uniform convergence on compact time intervals~\cite[Lemma~3.1]{benaim2005stochastic}.
\end{itemize}

\emph{Step 1: $V$ is constant on $L$.} Intuitively, a chain cannot climb from the minimum to the maximum of $V$, since $V$ decreases along each segment and the jumps are small. Since $L$ is compact and $V$ is continuous on $L$, $v=\min_LV$ and $\bar v=\max_LV$ are attained. Suppose, for contradiction, that $v<\bar v$. Since $V(\Lambda)$ has empty interior, we can choose a level
\[
u\in(v,\bar v)\quad\text{with}\quad u\notin V(\Lambda).
\]

\emph{Step 1a: segments in $L$ cross the level $u$ with a margin.}

\textbf{\boldmath Claim.} There is $\kappa>0$ such that every solution $z$ satisfies
\[
\left.\begin{aligned}&z([0,t])\subseteq L,\ t\ge1,\\&V(z(0))\le u+\kappa\end{aligned}\right\}\ \Longrightarrow\ V(z(t))\le u-\kappa.
\]
That is, a segment of length at least $1$ that starts near the level $u$ ends clearly below it.

\textit{Proof of the claim.} Suppose not. Then for every $n$ there are $z_n$ and $t_n\ge1$ with $z_n([0,t_n])\subseteq L$, $V(z_n(0))\le u+1/n$ and $V(z_n(t_n))>u-1/n$. By (F1),
\[
u-1/n<V(z_n(1))\le V(z_n(0))\le u+1/n.
\]
By (F2), a subsequence converges uniformly on $[0,1]$ to a solution $z$. Since $L$ is closed and $V$ is continuous on $L$, $z([0,1])\subseteq L$ and $V(z(0))=V(z(1))=u$. But $u\notin V(\Lambda)$, so $z(0)\notin\Lambda$, and the hypothesis gives $V(z(1))<V(z(0))$, a contradiction.\hfill$\triangleleft$

\emph{Step 1b: no chain climbs from $v$ to $\bar v$.} Let
\[
\eta(\varepsilon)=\sup\{|V(p)-V(q)|:p,q\in L,\ \norm{p-q}\le\varepsilon\},
\]
which tends to $0$ as $\varepsilon\to0$ because $V$ is uniformly continuous on the compact set $L$. Choose $x,y\in L$ with $V(x)=v$ and $V(y)=\bar v$, choose $\varepsilon$ with $\eta(\varepsilon)<\min(\kappa,u-v)$, and take an $(\varepsilon,1)$-chain from $x$ to $y$. All points $x_i(0)$ and $x_i(t_i)$ lie in $L$, so each jump of the chain changes $V$ by at most $\eta(\varepsilon)$. By induction, every segment of the chain ends at or below $u-\kappa$:
\begin{itemize}
    \item the chain starts below $u$: $V(x_1(0))\le v+\eta(\varepsilon)<u$;
    \item if $V(x_i(0))\le u+\kappa$, Step~1a gives $V(x_i(t_i))\le u-\kappa$, hence $V(x_{i+1}(0))\le u-\kappa+\eta(\varepsilon)<u$.
\end{itemize}
In particular $V(x_n(t_n))\le u-\kappa$, so
\[
\bar v=V(y)\le u-\kappa+\eta(\varepsilon)<u<\bar v,
\]
a contradiction. Hence $V$ is constant on $L$.

\emph{Step 2: $L\subseteq\Lambda$.} Let $x\in L$. By Lemma~\ref{lem:bhs}(b), a solution through $x$ stays in $L$. If $x\notin\Lambda$, $V$ decreases strictly along it, which is impossible since $V$ is constant on $L$. Hence $x\in\Lambda$, so $L\subseteq\Lambda$.\hfill$\square$

\begin{lemma}[Martingale difference noise]
\label{lem:martingale_proof}
Let $h^{(N)}(\Theta):=\E_\omega[G(\Theta,\omega)]$, where $\omega$ has the law of Assumption~\ref{ass:sampling} with batch size $N$, and let $M_{k+1}:=\tilde G_k-h^{(N_k)}(\Theta_k)$. Then:
\begin{enumerate}
  \item $\E[M_{k+1}\mid\mathcal{F}_k]=\mathbf 0$ a.s.\ for all $k\ge 0$;
  \item for every $R>0$, on $\{\norm{\Theta_k}\le R\}$,
  \[
    \E\big[\norm{M_{k+1}}^2 \mid \mathcal{F}_k\big] \le C_R := 32|\Phi|^2(1+2K_{1,R})^2K_{2,R}^2 .
  \]
\end{enumerate}
\end{lemma}

\begin{proof}
\textbf{Part 1: martingale difference property.}
Given $\mathcal{F}_k$, the only randomness in $\tilde G_k$ is the fresh sample $\omega_k$. $\Theta_k$ is $\mathcal{F}_k$-measurable, and by Assumption~\ref{ass:sampling} the conditional law of $\omega_k$ given $\mathcal{F}_k$ is the fixed sampling law with batch size $N_k$. Hence $\E[\tilde G_k\mid\mathcal{F}_k]=\E_\omega[G(\Theta,\omega)]\big|_{\Theta=\Theta_k}=h^{(N_k)}(\Theta_k)$, which is $\mathcal{F}_k$-measurable, \textbf{\boldmath and $\E[M_{k+1}\mid\mathcal{F}_k]=\mathbf 0$}.

\textbf{Part 2: bounded conditional second moments.}
It suffices to bound the size of the update itself. Since $h^{(N_k)}(\Theta_k)=\E[\tilde G_k\mid\mathcal{F}_k]$,
\begin{align}
\E\big[\norm{M_{k+1}}^2\mid\mathcal{F}_k\big]
&= \E\big[\norm{\tilde G_k}^2\mid\mathcal{F}_k\big]-\norm{h^{(N_k)}(\Theta_k)}^2 \nonumber\\
&\le \E\big[\norm{\tilde G_k}^2\mid\mathcal{F}_k\big].
\end{align}
For an active block $j\in\{i_k,c_k\}$, write~\eqref{eq:avg_gradient} as $\tilde g_j = 2\sum_{\phi}\bar\delta_j(\phi)\,\bar D_j(\phi)$, with $\bar D_j(\phi)$ as in Section~\ref{sec:loss}. On $\{\norm{\Theta_k}\le R\}$, since averages obey the same bounds as their summands and $\gamma_p\le1$,
\begin{align}
|\bar\delta_j(\phi)| &\le 1+2K_{1,R}, &
\norm{\bar D_j(\phi)} &\le 2K_{2,R}.
\end{align}
Hence, by the triangle inequality,
\begin{equation}
\norm{\tilde g_j} \le 2|\Phi|(1+2K_{1,R})(2K_{2,R}) = 4|\Phi|(1+2K_{1,R})K_{2,R}.
\end{equation}
The blocks of $\tilde G_k$ are disjoint and at most two are non-zero, so
\begin{align}
\norm{\tilde G_k}^2 &\le 2\max_{j\in\{i_k,c_k\}}\norm{\tilde g_j}^2 \nonumber\\
&\le 32|\Phi|^2(1+2K_{1,R})^2K_{2,R}^2 = C_R .
\end{align}
Taking conditional expectations completes the proof.
\end{proof}

\begin{lemma}[Mean-field identification and batch bias]
\label{lemma:mean-field}
Under Assumptions~\ref{ass:regularity} and~\ref{ass:sampling}, let $h^{(N)}$ be as in Lemma~\ref{lem:martingale_proof}, and let $h^{(\infty)}(\Theta)$ be the population field obtained by replacing the sample averages in~\eqref{eq:avg_loss}--\eqref{eq:avg_gradient} with conditional expectations under $P(\cdot\mid s,a)$, using the same tie-breaking rule. Then:
\begin{enumerate}
\item[(a)] $\E[\tilde G_k\mid\mathcal{F}_k]=h^{(N_k)}(\Theta_k)$, with $h^{(N)}(\Theta)\in\Hbar^{(N)}(\Theta)$ and $h^{(\infty)}(\Theta)\in\Hbar^{(\infty)}(\Theta)$ for all $\Theta$;
\item[(b)] for all $\Theta$, $E^{(N)}(\Theta)=\Ebe(\Theta)+V(\Theta)/N$ (Proposition~\ref{prop:decomposition}). Moreover, $E^{(N)},\Ebe,V\in C^2(\Ureg)$, and on $\Ureg$ we have $\nabla E^{(N)}=h^{(N)}$ and $\nabla\Ebe=h^{(\infty)}$, so that $\Hbar^{(N)}=\{\nabla E^{(N)}\}$ and $\Hbar^{(\infty)}=\{\nabla\Ebe\}$ there;
\item[(c)] for every $R>0$ and $N\ge1$, $\sup_{\norm{\Theta}\le R}\norm{h^{(N)}(\Theta)-h^{(\infty)}(\Theta)}\le C'_R/N$, where $C'_R:=16|\Phi|(1+K_{1,R})K_{2,R}$.
\end{enumerate}
\end{lemma}

\begin{proof}
\textbf{Step 1: conditioning.}
Part (a) for $h^{(N_k)}$ is Part~1 of the proof of Lemma~\ref{lem:martingale_proof}. For every $\omega$, $G(\Theta,\omega)=G^{\sigma}(\Theta,\omega)$ for the selection $\sigma\in\Sigma(\Theta,\omega)$ produced by the tie-breaking rule. Hence $G(\Theta,\omega)\in F^{(N)}(\Theta,\omega)$, and taking expectations gives $h^{(N)}(\Theta)\in\Hbar^{(N)}(\Theta)$. The same argument applies to $h^{(\infty)}$.

\textbf{Step 2: covariance identity.}
The averaged update multiplies two sample means, and the mean of their product differs from the product of their means by a covariance divided by $N$. Fix $(i,s,a)$, $\Theta$, an active block $j\in\{i,c\}$ with $c=c(\Theta,s,i)$, and $\phi\in\Phi$. Let $Y_p:=\ind(\phi=\phi_p)+\gamma_p\,\xi_j(s'_p,b_{j,p},\phi;\theta^{(j)})$ and $D_p:=\gamma_p\nabla\xi_j(s'_p,b_{j,p},\phi;\theta^{(j)})$, and let $x:=\xi_j(s,a,\phi;\theta^{(j)})$ and $y:=\nabla\xi_j(s,a,\phi;\theta^{(j)})$. The prior $c$ depends only on $(\Theta,s,i)$, and $b_{j,p}=b_j(\Theta,s'_p)$ depends only on $(\Theta,s'_p)$. Hence the pairs $(Y_p,D_p)$ are i.i.d.\ given $(i,s,a)$, with means $(m_Y,m_D)$. For the sample means $\bar Y$ and $\bar D$,
\begin{align}
\E\big[(\bar Y-x)(\bar D-y)\big] &= (m_Y-x)(m_D-y) \nonumber\\
&\quad + \tfrac{1}{N}\mathrm{Cov}(Y_1,D_1), \label{eq:cov}\\
\E\big[(\bar Y-x)^2\big] &= (m_Y-x)^2 + \tfrac{1}{N}\mathrm{Var}(Y_1). \label{eq:var}
\end{align}
The first terms on the right are the population quantities, since
\begin{multline*}
m_Y=\E_{s'\sim P(\cdot\mid s,a)}\big[\ind(\phi=\phi(s,a,s')) \\
+\gamma(s')\,\xi_j(s',b_j(\Theta,s'),\phi;\theta^{(j)})\big]
\end{multline*}
is the expected Bellman target of block $j$, which replaces the per-sample target in $\Ebe$.

\textbf{Step 3: decomposition and gradient.}
Summing~\eqref{eq:var} over $\phi$ and over the active blocks and taking $\E_{i\sim\rho}\E_{(s,a)\sim\mu_i}$ gives $E^{(N)}=\Ebe+V/N$ for every $\Theta$. Fix $\Theta\in\Ureg$. On a neighborhood of $\Theta$, all selections $c$ and $b_j(\cdot)$ are constant. Because the spaces are finite, $E^{(N)}$ is there a finite sum of $C^2$ functions of $\Theta$, which may be differentiated term by term. Moreover, $\ell_i$ depends only on $\theta^{(i)}$ and $\ell_c$ only on $\theta^{(c)}$. Hence $\nabla_{\theta^{(j)}}E^{(N)}=\E[G_j]$, i.e., $\nabla E^{(N)}=h^{(N)}$. The same argument applied to $\Ebe$ gives $\nabla\Ebe=h^{(\infty)}$. Since $\Sigma$ is a singleton near $\Theta$, $\Hbar^{(N)}(\Theta)=\{h^{(N)}(\Theta)\}$ and $\Hbar^{(\infty)}(\Theta)=\{h^{(\infty)}(\Theta)\}$.

\textbf{Step 4: batch bias.}
By~\eqref{eq:cov}, the block-$j$ components of $h^{(N)}$ and $h^{(\infty)}$ differ by $\E_{i,s,a}\big[\frac{2}{N}\sum_{\phi}\mathrm{Cov}(Y_1,D_1)\big]$. On $\{\norm{\Theta}\le R\}$ we have $|Y_1-m_Y|\le2(1+K_{1,R})$ and $\norm{D_1-m_D}\le2K_{2,R}$, so $\norm{\mathrm{Cov}(Y_1,D_1)}\le4(1+K_{1,R})K_{2,R}$. Summing over $|\Phi|$ features and at most two blocks yields (c). In particular, the bias is $O(1/N)$ rather than $O(N^{-1/2})$: the $O(N^{-1/2})$ fluctuations of $\bar Y$ and $\bar D$ have conditional mean zero and enter $M_{k+1}$, and only their covariance shifts the mean field.
\end{proof}

\begin{remark}[Block structure of the mean field]\label{rem:block}
The $j$-th block of $h^{(N)}$ is
\begin{multline*}
\rho(j)\,\E_{\mu_j}\big[\nabla_{\theta^{(j)}}\ell_j\big] \\
+\sum_{i\neq j}\rho(i)\,\E_{\mu_i}\big[\ind\{c(\Theta,s,i)=j\}\nabla_{\theta^{(j)}}\ell_j\big].
\end{multline*}
Block $j$ thus aggregates two losses with \emph{different} targets: the GPI target when $j$ is the current task, and its own greedy target, on the states where GPI selects it, when $j$ is a prior. It is therefore not a multiple of the gradient of a single per-task loss. With $E$ defined by~\eqref{eq:objective}, however, it equals $\nabla_{\theta^{(j)}}E$ on $\Ureg$, which is all the proof of Theorem~\ref{thm:main} requires.
\end{remark}

\clearpage
\section*{Appendix B: Experimental Details}

\subsection*{B.1 Environments}
\label{subsec:envs}
\paragraph*{Four Rooms~\cite{chevalier2023minigrid} (discrete features)}
The agent navigates a four-room grid world to reach a fixed goal while collecting objects. Each object yields a reward only when it is first collected.

\emph{Tasks.} The state concatenates one-hot grid coordinates with a binary inventory vector of collected objects. The feature vector $\phi(s, a, s') \in \{0, 1\}^{n_o+1}$ indicates whether an object of type $o \in \{1, \dots, n_o\}$ was collected or the goal was reached. Tasks are defined by reward weights $\mathbf{w}_j$ with $r = \phi^\top \mathbf{w}_j$, so that they assign different values to the objects.

\paragraph*{Reacher~\cite{coumans2016pybullet} (continuous features)}
We use the PyBullet implementation of a two-joint robotic arm that must reach a target location. The state consists of the joint angles and velocities, and the actions are 9 discretized joint torques.

\emph{Tasks.}
Each task corresponds to a different target location. The features $\phi$ combine the state with a proximity measure, so that the reward is approximately linear in the features.

\paragraph*{PointMaze~\cite{gymnasium_robotics2023github} (continuous features)}
A point agent navigates a 2D maze to reach a goal. The state consists of the agent's observation and the goal coordinates, and the continuous action space is discretized into 9 actions by mapping $\{-1,0,1\}^2$ to control inputs.

\emph{Tasks.}
Each task is defined by a distinct goal location $g_i$, with dense reward
\begin{equation}
    r_i(s,a,s') = \exp\left(-\|x(s') - g_i\|\right),
\end{equation}
where $x(s')$ denotes the agent's position. We evaluate on layouts of increasing complexity (\texttt{UMaze}, \texttt{Medium}, \texttt{Large}). In every layout, episodes start in cell $(3,3)$, and the 8 goals are cell $(1,1)$ and its nearest free cells, so that neighboring tasks are related. Since \texttt{UMaze} has only 7 free cells, its eighth goal is the same cell as the seventh.

\paragraph*{Linear successor features}
For Reacher and PointMaze, whose rewards are linear in $\phi$, we use the linear-SF parameterization $\psi_j(s,a)\in\R^{\dim\phi}$. The update carries over, with the sum over $\Phi$ in~\eqref{eq:avg_loss} replaced by the squared Euclidean norm of the vector residual $\phi(s,a,s')+\gamma(s')\psi_j(s',b)-\psi_j(s,a)$.

\subsection*{B.2 Hyperparameters}
Table~\ref{tab:hyperparams} lists the hyperparameters of the experiments of Section~\ref{sec:exp}.

\begin{table}[!htbp]
\centering
\caption{Experimental hyperparameters}
\label{tab:hyperparams}
\setlength{\tabcolsep}{3pt}
\begin{tabular}{lccc}
\hline
\textbf{Parameter} & \textbf{Four Rooms} & \textbf{Reacher} & \textbf{PointMaze} \\
\hline
Steps per task & 10,000 & 40,000 & 30,000 \\
Number of tasks       & 6      & 4        & 8      \\
Exploration ($\epsilon$) & 0.60 & 0.55 & 0.60 \\
\hline
Batch size            & \multicolumn{3}{c}{512} \\
Replay buffer size    & \multicolumn{3}{c}{200,000} \\
Discount ($\gamma$) & \multicolumn{3}{c}{0.95} \\
Episode horizon ($T$) & \multicolumn{3}{c}{200} \\
Learning rate & \multicolumn{3}{c}{0.001} \\
Reward-weight LR ($\alpha_R$) & \multicolumn{3}{c}{0.5} \\
Averaging samples ($N$) & \multicolumn{3}{c}{5 (ablation: 5, 10, 20)} \\
Seeds & \multicolumn{3}{c}{3} \\
\hline
\end{tabular}
\end{table}

\section*{Appendix C: Additional Experimental Results}

\subsection*{C.1 Training and evaluation totals}
\label{subsec:totals}
Tables~\ref{tab:train_totals} and~\ref{tab:final_eval} report the total training reward and the mean final evaluation return of all methods in the five domains of Section~\ref{sec:exp}.
\input{figures/table_train.tex}
\input{figures/table_eval.tex}

\subsection*{C.2 Slippery Four Rooms}
\label{subsec:slippery_fourrooms}
Table~\ref{tab:slippery} and Fig.~\ref{fig:slippery_totals} report the totals of the slippery Four Rooms experiment of Section~\ref{subsec:slippery}.

\begin{table}[!htbp]
\centering
\caption{Slippery Four Rooms: total training reward and final evaluation (mean\,$\pm$\,std over 3 seeds)}
\label{tab:slippery}
\setlength{\tabcolsep}{3pt}
\begin{tabular}{llcc}
\hline
$p$ & \textbf{Method} & \textbf{Training reward} & \textbf{Final evaluation} \\
\hline
0.2 & DQN & $942\pm60$ & $0.29\pm0.18$ \\
    & FG-DQN & $1{,}505\pm39$ & $0.00\pm0.00$ \\
    & SFDQN & $1{,}593\pm218$ & $0.41\pm0.20$ \\
    & Alg.~\ref{alg:fg-sfrql} & $3{,}410\pm105$ & $2.92\pm0.10$ \\
    & Alg.~\ref{alg:fg-sfrql}, averaged ($N=5$) & $\mathbf{4{,}043\pm122}$ & $\mathbf{3.79\pm0.16}$ \\
    & Alg.~\ref{alg:fg-sfrql-random} & $3{,}452\pm44$ & $3.17\pm0.18$ \\
    & Alg.~\ref{alg:fg-sfrql-avg}~($N=5$) & $3{,}866\pm57$ & $3.51\pm0.07$ \\
\hline
0.3 & DQN & $793\pm88$ & $0.31\pm0.29$ \\
    & FG-DQN & $1{,}177\pm68$ & $0.37\pm0.03$ \\
    & SFDQN & $1{,}531\pm55$ & $0.54\pm0.05$ \\
    & Alg.~\ref{alg:fg-sfrql} & $2{,}580\pm188$ & $2.67\pm0.07$ \\
    & Alg.~\ref{alg:fg-sfrql}, averaged ($N=5$) & $\mathbf{3{,}303\pm65}$ & $\mathbf{3.19\pm0.19}$ \\
    & Alg.~\ref{alg:fg-sfrql-random} & $2{,}838\pm96$ & $2.86\pm0.13$ \\
    & Alg.~\ref{alg:fg-sfrql-avg}~($N=5$) & $3{,}256\pm78$ & $3.13\pm0.26$ \\
\hline
\end{tabular}
\end{table}

\begin{figure}[!htbp]
\centering
\begin{subfigure}[t]{\linewidth}
\centering
\includegraphics[width=\linewidth]{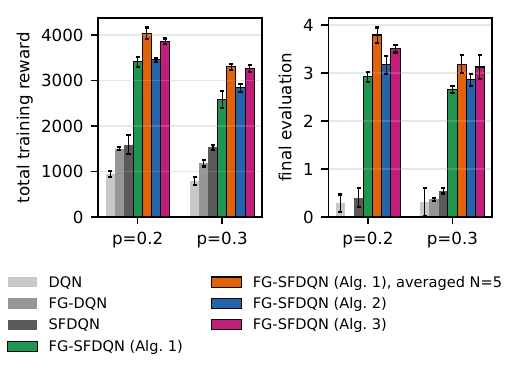}
\caption{Total training reward and final evaluation}
\end{subfigure}

\medskip
\begin{subfigure}[t]{\linewidth}
\centering
\includegraphics[width=\linewidth]{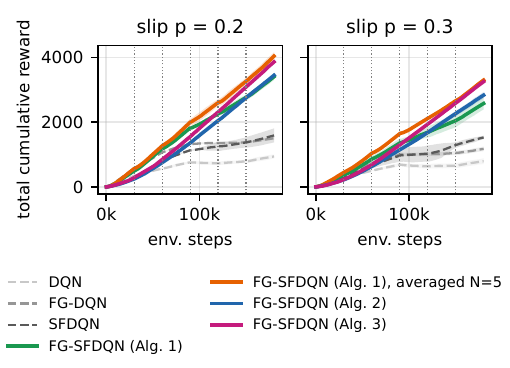}
\caption{Total cumulative training reward over all tasks}
\end{subfigure}
\caption{Slippery Four Rooms: totals (mean\,$\pm$\,std over 3 seeds). Dotted lines in~(b) mark the task switches of the sequential methods.}
\label{fig:slippery_totals}
\end{figure}

\subsection*{C.3 Slippery Reacher}
\label{subsec:slippery_reacher}
Table~\ref{tab:slippery_reacher} and Fig.~\ref{fig:slippery_reacher} report the totals of the slippery Reacher experiment of Section~\ref{subsec:slippery}; the per-task results are in Fig.~\ref{fig:slippery}. The final evaluation is the mean return over 10 episodes of at most 100 steps on each of the 4 tasks.

\begin{table}[!htbp]
\centering
\caption{Slippery Reacher: total training reward and final evaluation (mean\,$\pm$\,std over 3 seeds)}
\label{tab:slippery_reacher}
\setlength{\tabcolsep}{3pt}
\begin{tabular}{llcc}
\hline
$p$ & \textbf{Method} & \textbf{Training reward} & \textbf{Final evaluation} \\
\hline
0.2 & DQN & $66{,}789\pm4{,}291$ & $32.7\pm5.5$ \\
    & FG-DQN & $58{,}239\pm10{,}760$ & $37.2\pm2.9$ \\
    & SFDQN & $\mathbf{82{,}470\pm3{,}750}$ & $42.3\pm10.6$ \\
    & Alg.~\ref{alg:fg-sfrql} & $71{,}634\pm1{,}072$ & $57.5\pm1.6$ \\
    & Alg.~\ref{alg:fg-sfrql}, averaged ($N=5$) & $80{,}967\pm4{,}219$ & $\mathbf{69.3\pm1.8}$ \\
    & Alg.~\ref{alg:fg-sfrql-random} & $63{,}949\pm4{,}347$ & $53.5\pm1.4$ \\
    & Alg.~\ref{alg:fg-sfrql-avg}~($N=5$) & $73{,}018\pm3{,}223$ & $68.1\pm4.5$ \\
\hline
0.3 & DQN & $61{,}410\pm4{,}035$ & $34.7\pm4.6$ \\
    & FG-DQN & $55{,}988\pm2{,}682$ & $38.5\pm3.9$ \\
    & SFDQN & $\mathbf{76{,}853\pm6{,}729}$ & $46.6\pm3.3$ \\
    & Alg.~\ref{alg:fg-sfrql} & $67{,}698\pm561$ & $52.6\pm0.8$ \\
    & Alg.~\ref{alg:fg-sfrql}, averaged ($N=5$) & $74{,}379\pm2{,}362$ & $61.7\pm3.4$ \\
    & Alg.~\ref{alg:fg-sfrql-random} & $63{,}241\pm2{,}779$ & $48.1\pm5.1$ \\
    & Alg.~\ref{alg:fg-sfrql-avg}~($N=5$) & $67{,}267\pm1{,}077$ & $\mathbf{63.0\pm0.9}$ \\
\hline
\end{tabular}
\end{table}

\begin{figure}[!htbp]
\centering
\begin{subfigure}[t]{\linewidth}
\centering
\includegraphics[width=\linewidth]{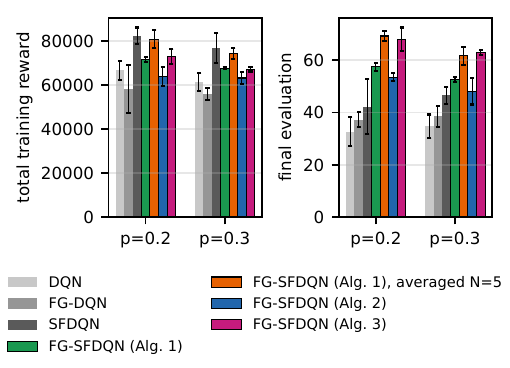}
\caption{Total training reward and final evaluation (mean\,$\pm$\,std over 3 seeds)}
\end{subfigure}

\medskip
\begin{subfigure}[t]{\linewidth}
\centering
\includegraphics[width=\linewidth]{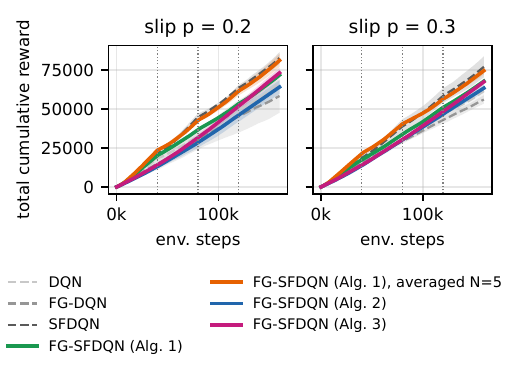}
\caption{Total cumulative training reward over all tasks}
\end{subfigure}
\caption{Slippery Reacher totals: DQN, FG-DQN, SFDQN, Algorithm~\ref{alg:fg-sfrql} and its averaged variant ($N=5$), and Algorithms~\ref{alg:fg-sfrql-random} and~\ref{alg:fg-sfrql-avg} (mean\,$\pm$\,std over 3 seeds).}
\label{fig:slippery_reacher}
\end{figure}

\subsection*{C.4 Optimization diagnostics}
\label{subsec:diagnostics}
During training, we periodically measure two quantities on replay batches using the online networks:
\begin{itemize}
    \item the squared SF Bellman residual of the current task with the GPI target action, i.e., the per-sample objective of FG-SFRQL;
    \item the fraction of GPI target actions decided by a margin below $10^{-3}$, i.e., near-ties.
\end{itemize}
Both are shown in Fig.~\ref{fig:diagnostics}, and the final residuals are listed in Table~\ref{tab:residual}.

The full-gradient variants attain a much smaller Bellman residual than the semi-gradient SFDQN. At the end of training, the residual of Algorithm~\ref{alg:fg-sfrql} is $46\times$ (Four Rooms), $22\times$ (Reacher) and $5.2$--$6.8\times$ (all three PointMaze layouts) smaller than that of SFDQN. The randomized variants also improve on SFDQN, by $19$--$20\times$ in Four Rooms, $17$--$18\times$ in Reacher and $2.1$--$2.4\times$ in PointMaze-UMaze and -Medium. In PointMaze-Large, Algorithm~\ref{alg:fg-sfrql-random} improves on SFDQN by $1.2\times$, while both Algorithms~\ref{alg:fg-sfrql-random} and~\ref{alg:fg-sfrql-avg} still surpass SFDQN in final evaluation return ($20.12$ and $22.02$ vs.\ $19.54$). For Algorithm~\ref{alg:fg-sfrql}, the residual spikes briefly at task switches when a new network starts training, and then decays rapidly.

This substantial reduction in Bellman error directly validates our theoretical formulation: differentiating through the bootstrapped target effectively suppresses the SF Bellman residual where semi-gradient updates stall, directly translating into higher final evaluation returns across all domains (Table~\ref{tab:final_eval}).

Near-ties in the GPI target action occur primarily at the beginning of training, when all value estimates are close to zero. By the end of training, they affect only $1.5$--$8.6\%$ of the next states for the full-gradient variants, and ties in the choice of the GPI prior affect at most $1.5\%$ of the states, providing strong empirical support for the tie-free condition of Theorem~\ref{thm:main}(ii).

\begin{figure*}[!t]
\centering
\includegraphics[width=0.85\textwidth]{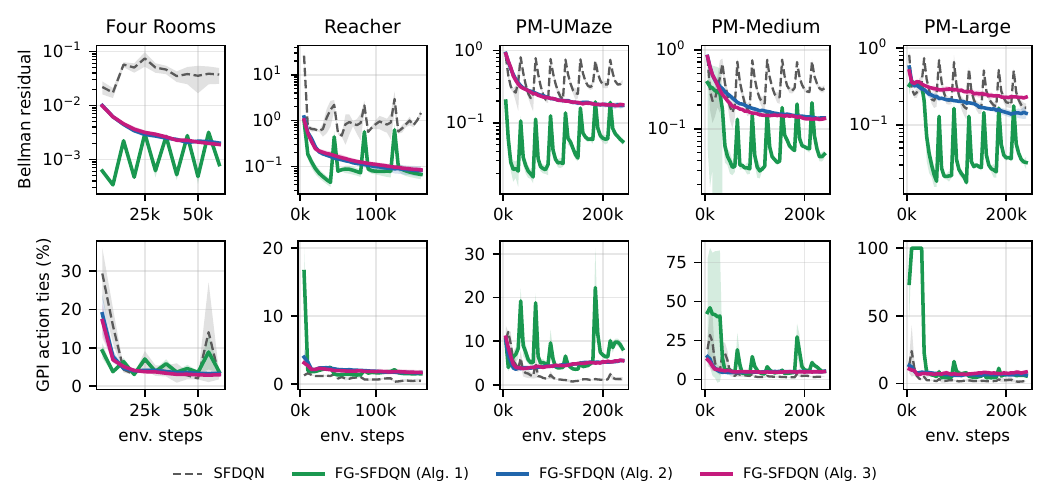}
\caption{Optimization diagnostics (mean\,$\pm$\,std over 3 seeds). Top: squared SF Bellman residual of the current task (log scale). Bottom: percentage of GPI target actions decided by a margin below $10^{-3}$.}
\label{fig:diagnostics}
\end{figure*}

\begin{table*}[!t]
\centering
\caption{Final squared SF Bellman residual (mean over 3 seeds; lower is better)}
\label{tab:residual}
\setlength{\tabcolsep}{3pt}
\begin{tabular}{lccccc}
\hline
\textbf{Method} & \textbf{4R} & \textbf{Reacher} & \textbf{PM-U} & \textbf{PM-M} & \textbf{PM-L} \\
\hline
SFDQN & 0.0373 & 1.468 & 0.374 & 0.325 & 0.167 \\
Alg.~\ref{alg:fg-sfrql} & \textbf{0.0008} & \textbf{0.067} & \textbf{0.055} & \textbf{0.048} & \textbf{0.032} \\
Alg.~\ref{alg:fg-sfrql-random} & 0.0020 & 0.081 & 0.174 & 0.138 & 0.142 \\
Alg.~\ref{alg:fg-sfrql-avg} & 0.0019 & 0.084 & 0.178 & 0.135 & 0.230 \\
\hline
\end{tabular}
\end{table*}

\subsection*{C.5 Averaging ablation}
\label{subsec:ablation_averaging}
Figure~\ref{fig:ablation_avg} compares Algorithm~\ref{alg:fg-sfrql} with its averaged variant, which averages the Bellman target over $N\in\{5,10,20\}$ next states sampled from the same state-action pair, with per-sample GPI actions (Section~\ref{sec:loss}). The total training rewards (mean\,$\pm$\,std over 3 seeds) are:
\begin{itemize}
    \item Four Rooms: $1365\pm185$ for Algorithm~\ref{alg:fg-sfrql}, and $1333\pm174$, $1368\pm36$ and $1249\pm22$ for $N=5,10,20$;
    \item Reacher: $85{,}328\pm4{,}935$ for Algorithm~\ref{alg:fg-sfrql}, and $85{,}251\pm3{,}143$, $82{,}845\pm3{,}529$ and $88{,}518\pm6{,}990$ for $N=5,10,20$.
\end{itemize}
Consistent with Proposition~\ref{prop:decomposition}, under deterministic transitions ($V=0$) the mean returns across $N$ all lie within one standard deviation of single-sample FG-SFDQN. Notably, larger sample sizes ($N=10$ and $N=20$) in Four Rooms reduce the standard deviation across seeds by up to $8\times$ ($22$ vs.\ $185$), stabilizing the learning trajectory.

\begin{figure}[!htbp]
\centering
\begin{subfigure}[t]{\linewidth}
\centering
\includegraphics[width=0.78\linewidth]{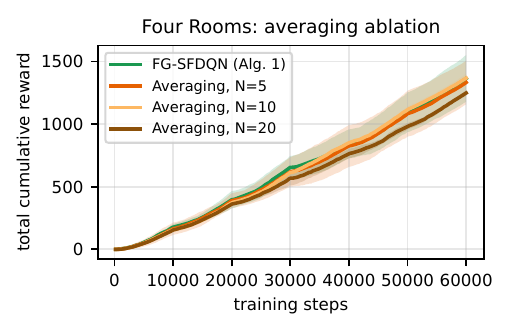}
\caption{Four Rooms}
\end{subfigure}

\medskip
\begin{subfigure}[t]{\linewidth}
\centering
\includegraphics[width=0.78\linewidth]{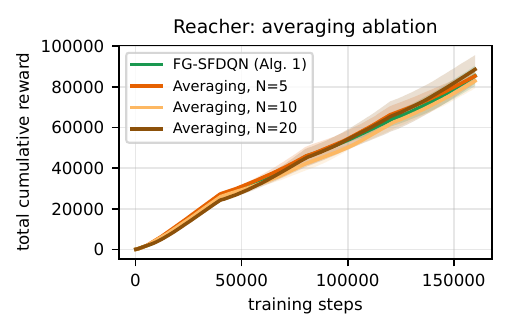}
\caption{Reacher}
\end{subfigure}
\caption{Averaging ablation: total cumulative training reward of Algorithm~\ref{alg:fg-sfrql} and its averaged variant with $N=5,10,20$ next-state samples (mean\,$\pm$\,std over 3 seeds).}
\label{fig:ablation_avg}
\end{figure}

\subsection*{C.6 Computational overhead}
\label{subsec:overhead}
Table~\ref{tab:performance} reports the wall-clock time per training step on PointMaze-UMaze (8 tasks, batch size 512, NVIDIA RTX 4000 Ada GPU, Intel Core i9-14900KS CPU; 1{,}500 timed steps after a 512-step warm-up in each of 3 seeds). The SF methods are about $4\times$ (Algorithm~\ref{alg:fg-sfrql}, SFDQN) and $9\times$ (randomized variants) slower than the DQN baselines, since they evaluate the successor features of every policy for GPI. Within the SF family, replacing the semi-gradient update with the full gradient incurs no computational penalty: Algorithm~\ref{alg:fg-sfrql} matches SFDQN in mean step time ($4.34$ vs.\ $4.47$\,ms) and achieves a lower median step time ($2.67$ vs.\ $4.19$\,ms). Furthermore, target averaging ($N=5$) in Algorithm~\ref{alg:fg-sfrql} adds modest overhead ($5.11$\,ms mean), while still maintaining a lower median step time ($3.85$\,ms) than SFDQN. The per-step runtime is primarily determined by the number of optimizer steps, one for the current task and one per distinct GPI-selected prior updated in the minibatch (correlation $0.97$--$1.00$). While a full-gradient backward pass involves additional differentiation through the target, a likely factor contributing to Algorithm~\ref{alg:fg-sfrql}'s lower median step time is that as policy preferences stabilize, fewer distinct priors may be triggered per batch. The randomized variants (Algorithms~\ref{alg:fg-sfrql-random} and~\ref{alg:fg-sfrql-avg}) take about twice as long ($9.59$ and $10.12$\,ms) because uniform task sampling interleaves transitions from all tasks across every batch.

\begin{table}[!htbp]
\centering
\caption{Per-step training time on PointMaze-UMaze}
\label{tab:performance}
\setlength{\tabcolsep}{4pt}
\input{figures/table_overhead.tex}
\end{table}

\clearpage

\onecolumn
\section*{Appendix D: Algorithms}
\setcounter{algorithm}{0}
\renewcommand{\theHalgorithm}{D.\arabic{algorithm}}
\begin{algorithm}[H]
\caption{Joint Full-Gradient Model-Free SFRQL (FG-SFRQL), full pseudocode}
\begin{algorithmic}[1]
\Require exploration rate $\epsilon$; step sizes $\{\alpha_t\}$ for the $\xi$-functions; learning rate for reward models $\alpha_R$
\Require features $\phi$ or $\Phi$; optional: reward functions for tasks $\{\mathcal{R}_1, \mathcal{R}_2, \dots, \mathcal{R}_{\text{num\_tasks}}\}$
\Statex
\For{$i = 1$ to num\_tasks}
    \If{$\mathcal{R}_i$ not provided}
        \State initialize $\tilde{\mathcal{R}}_i$: $\theta_i^R \leftarrow$ small random initial values
    \EndIf
    \If{$i = 1$}
        \State initialize $\xi_1$: $\theta_1^\xi \leftarrow$ small random initial values
    \Else
        \State $\theta_i^\xi \leftarrow \theta_{i-1}^\xi$
    \EndIf
    \State new\_episode $\leftarrow$ \textbf{true}
    \For{$t=1$ to num\_steps}
        \If{new\_episode}
            \State new\_episode $\leftarrow$ \textbf{false}
            \State $s_t \leftarrow$ initial state
        \EndIf
        \State $c \leftarrow \arg\max_{k \in \{1, \dots, i\}} \max_a \sum_{\phi \in \Phi} \xi_k(s_t, a, \phi; \theta_k^\xi) \tilde{\mathcal{R}}_i(\phi)$ \Comment{GPI prior at the current state}
        \State With probability $\epsilon$ select a random action $a_t$, otherwise $a_t \leftarrow \arg\max_a \sum_{\phi \in \Phi} \xi_c(s_t, a, \phi; \theta_c^\xi) \tilde{\mathcal{R}}_i(\phi)$

        \State Take action $a_t$, observe reward $r_t$ and next state $s_{t+1}$. Let $\phi_t = \phi(s_t, a_t, s_{t+1})$.

        \If{$\tilde{\mathcal{R}}_i$ not provided}
            \State Update $\theta_i^R$ using $\mathrm{SGD}(\alpha_R)$ with $\mathcal{L}_R = (r_t - \tilde{\mathcal{R}}_i(\phi_t))^2$
        \EndIf

        \If{$s_{t+1}$ is a terminal state}
            \State $\gamma_t \leftarrow 0$; new\_episode $\leftarrow$ \textbf{true}
        \Else
            \State $\gamma_t \leftarrow \gamma$
        \EndIf

        \Statex
        \State \Comment{GPI next action for task $i$}
        \State $\hat{a}_{t+1} \leftarrow \arg\max_{a'} \max_{k \in \{1, \dots, i\}} \sum_{\phi' \in \Phi} \xi_k(s_{t+1}, a', \phi'; \theta_k^\xi) \tilde{\mathcal{R}}_i(\phi')$

        \State \Comment{Full-gradient update for policy $i$ (the factor 2 of~\eqref{eq:avg_gradient} is absorbed into $\alpha_t$)}
        \State $\nabla_{\theta_i^\xi} \mathcal{L}_t^{(i)} \leftarrow \sum_{\phi \in \Phi} \left[ \left(\ind(\phi=\phi_t) + \gamma_t \xi_i(s_{t+1}, \hat{a}_{t+1}, \phi; \theta_i^\xi)\right) - \xi_i(s_t, a_t, \phi; \theta_i^\xi) \right] \times $
        \Statex \qquad \qquad \qquad \qquad $\left( \gamma_t \nabla_{\theta_i^\xi} \xi_i(s_{t+1}, \hat{a}_{t+1}, \phi; \theta_i^\xi) - \nabla_{\theta_i^\xi} \xi_i(s_t, a_t, \phi; \theta_i^\xi) \right)$
        \State $\theta_i^\xi \leftarrow \theta_i^\xi - \alpha_t \nabla_{\theta_i^\xi} \mathcal{L}_t^{(i)}$

        \If{$c \neq i$}
            \State \Comment{Greedy next action for the executed policy $c$ under its own reward}
            \State $\hat{a}^{c}_{t+1} \leftarrow \arg\max_{a'} \sum_{\phi' \in \Phi} \xi_c(s_{t+1}, a', \phi'; \theta_c^\xi) \tilde{\mathcal{R}}_c(\phi')$

            \State \Comment{Full-gradient update for policy $c$}
            \State $\nabla_{\theta_c^\xi} \mathcal{L}_t^{(c)} \leftarrow \sum_{\phi \in \Phi} \left[ \left(\ind(\phi=\phi_t) + \gamma_t \xi_c(s_{t+1}, \hat{a}^{c}_{t+1}, \phi; \theta_c^\xi)\right) - \xi_c(s_t, a_t, \phi; \theta_c^\xi) \right] \times $
            \Statex \qquad \qquad \qquad \qquad $\left( \gamma_t \nabla_{\theta_c^\xi} \xi_c(s_{t+1}, \hat{a}^{c}_{t+1}, \phi; \theta_c^\xi) - \nabla_{\theta_c^\xi} \xi_c(s_t, a_t, \phi; \theta_c^\xi) \right)$
            \State $\theta_c^\xi \leftarrow \theta_c^\xi - \alpha_t \nabla_{\theta_c^\xi} \mathcal{L}_t^{(c)}$
        \EndIf

        \State For all other $j \notin \{i, c\}$, $\theta_j^\xi \leftarrow \theta_j^\xi$.
        \State $s_t \leftarrow s_{t+1}$
    \EndFor
\EndFor
\end{algorithmic}
\end{algorithm}

\clearpage
\begin{algorithm}[H]
\caption{Full-Gradient SFRQL with Randomized Task Sampling}
\label{alg:fg-sfrql-random}
\begin{algorithmic}[1]
\Require exploration rate $\epsilon$; step sizes $\{\alpha_k\}$; reward learning rate $\alpha_R$; task distribution $\rho$; pivot distributions $\{\mu_i\}$
\Require features $\Phi$; set of tasks $\{1,\dots,m\}$; optional task rewards $\{\mathcal{R}_i\}$
\State Initialize $\{\theta_j^\xi\}_{j=1}^m$ (successor-feature parameters) and $\{\theta_j^R\}_{j=1}^m$ (reward models)
\Statex
\For{$k=0,1,2,\dots$}
    \State Sample a task index $i \sim \rho$ \Comment{current task}
    \State Sample a pivot $(s,a) \sim \mu_i$ (generative model or replay buffer)
    \State Sample $s' \sim P(\cdot\mid s,a)$; observe reward $r$ and feature $\phi_k = \phi(s,a,s')$
    \If{$\mathcal{R}_i$ not provided}
        \State Update reward model $\tilde{\mathcal{R}}_i$ using $\mathrm{SGD}(\alpha_R)$ on loss $(r - \tilde{\mathcal{R}}_i(\phi_k))^2$
    \EndIf
    \State $\gamma_k \leftarrow 0$ if $s'$ is terminal, else $\gamma_k \leftarrow \gamma$
    \Statex
    \State \Comment{GPI prior at the pivot state}
    \State $c \leftarrow \arg\max_{j \in \{1,\dots,m\}} \max_{a''} \sum_{\phi \in \Phi} \xi_j(s,a'',\phi;\theta_j^\xi)\,\tilde{\mathcal{R}}_i(\phi)$
    \Statex
    \State \Comment{Update for current task $i$ (GPI target action; the factor 2 of~\eqref{eq:avg_gradient} is absorbed into $\alpha_k$)}
    \State $\hat{a}^{i} \leftarrow \arg\max_{a'} \max_{j\in\{1,\dots,m\}} \sum_{\phi'\in\Phi} \xi_j(s',a',\phi';\theta_j^\xi)\tilde{\mathcal{R}}_i(\phi')$
    \State $\nabla_{\theta_i^\xi}\mathcal{L}_k^{(i)} \leftarrow
        \sum_{\phi\in\Phi} \Big[ \ind(\phi=\phi_k) + \gamma_k \xi_i(s',\hat{a}^{i},\phi;\theta_i^\xi) - \xi_i(s,a,\phi;\theta_i^\xi) \Big]
        \times \Big(\gamma_k\nabla_{\theta_i^\xi}\xi_i(s',\hat{a}^{i},\phi;\theta_i^\xi) - \nabla_{\theta_i^\xi}\xi_i(s,a,\phi;\theta_i^\xi)\Big)$
    \Statex
    \If{$c \neq i$}
        \State \Comment{Update for GPI-selected prior $c$ (own greedy target action)}
        \State $\hat{a}^{c} \leftarrow \arg\max_{a'} \sum_{\phi'\in\Phi} \xi_c(s',a',\phi';\theta_c^\xi)\tilde{\mathcal{R}}_c(\phi')$
        \State $\nabla_{\theta_c^\xi}\mathcal{L}_k^{(c)} \leftarrow
            \sum_{\phi\in\Phi} \Big[ \ind(\phi=\phi_k) + \gamma_k \xi_c(s',\hat{a}^{c},\phi;\theta_c^\xi) - \xi_c(s,a,\phi;\theta_c^\xi) \Big]
            \times \Big(\gamma_k\nabla_{\theta_c^\xi}\xi_c(s',\hat{a}^{c},\phi;\theta_c^\xi) - \nabla_{\theta_c^\xi}\xi_c(s,a,\phi;\theta_c^\xi)\Big)$
        \State $\theta_c^\xi \leftarrow \theta_c^\xi - \alpha_k \nabla_{\theta_c^\xi}\mathcal{L}_k^{(c)}$
    \EndIf
    \State $\theta_i^\xi \leftarrow \theta_i^\xi - \alpha_k \nabla_{\theta_i^\xi}\mathcal{L}_k^{(i)}$ \Comment{all selections above use the pre-update $\Theta_k$}
    \Statex
    \State (If pivots come from a replay buffer: act $\epsilon$-greedily w.r.t.\ GPI in task $i$ and store the transition.)
\EndFor
\end{algorithmic}
\end{algorithm}

\begin{algorithm}[H]
\caption{Full-Gradient SFRQL with Randomized Task Sampling and Averaging}
\label{alg:fg-sfrql-avg}
\begin{algorithmic}[1]
\Require exploration rate $\epsilon$; step sizes $\{\alpha_k\}$; batch sizes $\{N_k\}$; reward learning rate $\alpha_R$; task distribution $\rho$; pivot distributions $\{\mu_i\}$
\Require features $\Phi$; set of tasks $\{1,\dots,m\}$; optional task rewards $\{\mathcal{R}_i\}$
\State Initialize $\{\theta_j^\xi\}_{j=1}^m$, $\{\theta_j^R\}_{j=1}^m$ (and replay buffer $\mathcal{D} = \emptyset$ if used)
\Statex
\For{$k=0,1,2,\dots$}
    \State Sample task $i \sim \rho$ and pivot $(s, a) \sim \mu_i$; let $N \leftarrow N_k$
    \State Draw $N$ next states $\mathcal{K} = \{s'_p\}_{p=1}^N$ i.i.d.\ from $P(\cdot\mid s,a)$ (or from transitions in $\mathcal{D}$ that start at $(s,a)$)
    \If{fewer than $N$ transitions are available} \textbf{continue} \EndIf
    \State $\phi_p \leftarrow \phi(s,a,s'_p)$, $\gamma_p \leftarrow \gamma(s'_p)$ for $p=1,\dots,N$
    \If{$\mathcal{R}_i$ not provided}
        \State Update reward model $\tilde{\mathcal{R}}_i$ on the observed rewards using $\mathrm{SGD}(\alpha_R)$, as in Algorithm~\ref{alg:fg-sfrql-random}
    \EndIf
    \Statex
    \State \Comment{GPI prior at the pivot and per-sample greedy target actions}
    \State $c \leftarrow \arg\max_{j} \max_{a''} \sum_{\phi} \xi_j(s,a'',\phi;\theta_j^\xi)\tilde{\mathcal{R}}_i(\phi)$
    \State $b_{i,p} \leftarrow \arg\max_{a'} \max_{j} \sum_{\phi'} \xi_j(s'_p,a',\phi';\theta_j^\xi)\tilde{\mathcal{R}}_i(\phi')$ for each $p$
    \State $b_{c,p} \leftarrow \arg\max_{a'} \sum_{\phi'} \xi_c(s'_p,a',\phi';\theta_c^\xi)\tilde{\mathcal{R}}_c(\phi')$ for each $p$ (if $c\neq i$)
    \Statex
    \State \Comment{Averaged residuals and feature gradients, Eqs.~\eqref{eq:avg_loss}--\eqref{eq:avg_gradient}, for $j \in \{i, c\}$; the factor 2 is absorbed into $\alpha_k$}
    \State $\bar\delta_j(\phi) \leftarrow \frac{1}{N} \sum_{p=1}^N \big[\ind(\phi=\phi_p) + \gamma_p\,\xi_j(s'_p, b_{j,p}, \phi; \theta_j^\xi)\big] - \xi_j(s, a, \phi; \theta_j^\xi)$ for all $\phi\in\Phi$
    \State $\bar D_j(\phi) \leftarrow \frac{1}{N} \sum_{p=1}^N \gamma_p\nabla \xi_j(s'_p, b_{j,p}, \phi; \theta_j^\xi) - \nabla \xi_j(s, a, \phi; \theta_j^\xi)$ for all $\phi\in\Phi$
    \Statex
    \State $\theta_i^\xi \leftarrow \theta_i^\xi - \alpha_k \sum_{\phi} \bar\delta_i(\phi)\, \bar D_i(\phi)$ \Comment{current task}
    \If{$c \neq i$}
        \State $\theta_c^\xi \leftarrow \theta_c^\xi - \alpha_k \sum_{\phi} \bar\delta_c(\phi)\, \bar D_c(\phi)$ \Comment{GPI-selected prior}
    \EndIf
    \Statex
    \State (If a replay buffer is used: act $\epsilon$-greedily w.r.t.\ GPI in task $i$ and store $(s, a, r, s', \phi, i)$ in $\mathcal{D}$.)
\EndFor
\end{algorithmic}
\end{algorithm}

\end{document}

%% file: figures/table_train.tex
\begin{table*}[t]
\centering
\caption{Total cumulative training reward over all tasks (mean\,$\pm$\,std over seeds).}
\label{tab:train_totals}
\setlength{\tabcolsep}{4pt}
\begin{tabular}{lccccc}
\hline
\textbf{Method} & \textbf{Four Rooms} & \textbf{Reacher} & \textbf{PointMaze-UMaze} & \textbf{PointMaze-Medium} & \textbf{PointMaze-Large} \\
\hline
DQN & 341\,$\pm$\,18 & 64996\,$\pm$\,7318 & 61325\,$\pm$\,14063 & 65897\,$\pm$\,5462 & 58929\,$\pm$\,1300 \\
FG-DQN & 611\,$\pm$\,80 & 63744\,$\pm$\,5407 & 63323\,$\pm$\,324 & 57534\,$\pm$\,2743 & 46333\,$\pm$\,2793 \\
SFDQN & 500\,$\pm$\,36 & 84135\,$\pm$\,6232 & 74735\,$\pm$\,5828 & 67948\,$\pm$\,3551 & 54621\,$\pm$\,5360 \\
\hline
FG-SFDQN (Alg.~1) & \textbf{1365\,$\pm$\,185} & \textbf{85328\,$\pm$\,4935} & \textbf{84737\,$\pm$\,2500} & \textbf{73746\,$\pm$\,3927} & \textbf{61866\,$\pm$\,1750} \\
FG-SFDQN (Alg.~2) & 990\,$\pm$\,96 & 79616\,$\pm$\,4551 & 80877\,$\pm$\,4144 & 68050\,$\pm$\,1248 & 51609\,$\pm$\,3526 \\
FG-SFDQN (Alg.~3) & 1081\,$\pm$\,26 & 78684\,$\pm$\,6687 & 82588\,$\pm$\,1349 & 67217\,$\pm$\,3439 & 51631\,$\pm$\,2033 \\
\hline
\end{tabular}
\end{table*}

%% file: figures/table_eval.tex
\begin{table*}[t]
\centering
\caption{Final evaluation: mean return per episode, averaged over tasks (10 episodes per task, 100-step cap; mean\,$\pm$\,std over seeds).}
\label{tab:final_eval}
\setlength{\tabcolsep}{4pt}
\begin{tabular}{lccccc}
\hline
\textbf{Method} & \textbf{Four Rooms} & \textbf{Reacher} & \textbf{PointMaze-UMaze} & \textbf{PointMaze-Medium} & \textbf{PointMaze-Large} \\
\hline
DQN & 0.03\,$\pm$\,0.04 & 33.92\,$\pm$\,5.08 & 18.98\,$\pm$\,3.41 & 10.87\,$\pm$\,1.53 & 15.01\,$\pm$\,2.04 \\
FG-DQN & 0.02\,$\pm$\,0.03 & 39.43\,$\pm$\,1.53 & 13.81\,$\pm$\,2.83 & 11.05\,$\pm$\,4.20 & 11.63\,$\pm$\,5.58 \\
SFDQN & 0.41\,$\pm$\,0.15 & 38.68\,$\pm$\,11.55 & 23.82\,$\pm$\,3.30 & 21.00\,$\pm$\,2.23 & 19.54\,$\pm$\,0.70 \\
\hline
FG-SFDQN (Alg.~1) & \textbf{3.38\,$\pm$\,0.03} & \textbf{71.71\,$\pm$\,5.70} & \textbf{29.54\,$\pm$\,0.37} & \textbf{27.10\,$\pm$\,1.25} & \textbf{27.24\,$\pm$\,0.71} \\
FG-SFDQN (Alg.~2) & 2.92\,$\pm$\,0.35 & 68.54\,$\pm$\,4.85 & 27.82\,$\pm$\,0.65 & 24.75\,$\pm$\,0.39 & 20.12\,$\pm$\,1.81 \\
FG-SFDQN (Alg.~3) & 3.20\,$\pm$\,0.31 & 68.05\,$\pm$\,5.46 & 28.54\,$\pm$\,0.94 & 25.92\,$\pm$\,1.81 & 22.02\,$\pm$\,3.36 \\
\hline
\end{tabular}
\end{table*}

%% file: figures/table_overhead.tex
\begin{tabular}{lccc}
\hline
\textbf{Algorithm} & \textbf{Mean (ms)} & \textbf{Median (ms)} & \textbf{Std (ms)} \\
\hline
DQN & 1.05 & 1.04 & 0.05 \\
FG-DQN & 1.19 & 1.18 & 0.04 \\
SFDQN & 4.47 & 4.19 & 1.22 \\
FG-SFDQN (Alg.~1) & 4.34 & 2.67 & 2.62 \\
FG-SFDQN (Alg.~1, $N=5$) & 5.11 & 3.85 & 2.83 \\
FG-SFDQN (Alg.~2) & 9.59 & 9.61 & 0.44 \\
FG-SFDQN (Alg.~3) & 10.12 & 10.12 & 0.46 \\
\hline
\end{tabular}